\documentclass[lettersize,journal, manuscript]{IEEEtran} 
\usepackage{lineno}
\usepackage{amsmath,amsfonts}
\usepackage{array} 
\usepackage{textcomp}
\usepackage{stfloats}
\usepackage{url}
\usepackage{verbatim}
\usepackage{graphicx}
\usepackage{cite}
\usepackage{latexsym}
\usepackage{microtype}

\usepackage{booktabs}
\usepackage{amsmath}
\usepackage{amssymb}
\usepackage{amsthm}

\newtheorem{theorem}{Theorem}

\newtheorem{remark}{Remark}

\usepackage{algorithm}
\usepackage{algpseudocode}
\usepackage{multirow}
\usepackage{mathtools}
\usepackage{subfigure}
\usepackage{graphicx}
\usepackage{balance}
\usepackage{color}
\usepackage{xcolor}
\usepackage{enumitem}
\usepackage{xspace}

\renewcommand{\vec}[1]{\ensuremath{\mathbf{#1}}}
\newcommand{\stitle}[1]{\vspace{1mm} \noindent {\bf #1}}

\newcommand{\ie}{{\it i.e.}}

\newcommand{\bN}{\ensuremath{\mathcal{N}}}

\newcommand{\bM}{\ensuremath{\mathcal{M}}}

\newcommand{\model}{THGNN}

\usepackage{caption}

\usepackage{type1cm} % scalable fonts
\usepackage{lettrine}

\makeatletter
\patchcmd{\@makecaption}
  {\scshape}
  {}
  {}
  {}
\makeatother

\begin{document}
\markboth{IEEE Transactions on Neural Networks and Learning Systems} 
{\MakeLowercase{\textit{}}:}

% \title{A Sample Article Using IEEEtran.cls\\ for IEEE Journals and Transactions}
% \title{Static and Conditional Prompt Tuning with Graph-Augmented Low-Resource Text Classification}
\title{Temporal and Heterogeneous Graph Neural Network for Remaining Useful Life Prediction }

\author{
    \IEEEauthorblockN{Zhihao Wen\IEEEauthorrefmark{3}\IEEEauthorrefmark{1},
    Pengcheng Wei\IEEEauthorrefmark{4}\IEEEauthorrefmark{1},
    Yuan Fang\IEEEauthorrefmark{3}\IEEEauthorrefmark{2}~\IEEEmembership{Member,~IEEE,}
    Fayao Liu\IEEEauthorrefmark{5},
    Zhenghua Chen\IEEEauthorrefmark{5}~\IEEEmembership{Senior Member,~IEEE}
    and Min Wu\IEEEauthorrefmark{5}~\IEEEmembership{Senior Member,~IEEE}
    }\\
    
    \thanks{\IEEEauthorblockA{\IEEEauthorrefmark{1}School of Computing and Information Systems,
    Singapore Management University,
    Singapore, 188065,
    Email: zhwen@smu.edu.sg, yfang@smu.edu.dg}}
    
    \thanks{\IEEEauthorblockA{\IEEEauthorrefmark{4}Information Systems Technology and Design Pillar, Singapore University of Technology and Design, Singapore, 
    Email: pengcheng\_wei@mymail.sutd.edu.sg}}
    
    \thanks{\IEEEauthorblockA{\IEEEauthorrefmark{5}Institute for Infocomm
Research, A∗STAR, Singapore,
    Email: liu\_fayao@i2r.a-star.edu.sg, chen0832@e.ntu.edu.sg, wumin@i2r.a-star.edu.sg }}
    
    \thanks{\IEEEauthorblockA{\IEEEauthorrefmark{1}Zhihao Wen and Pengcheng Wei contributed equally to this work.}}
    \thanks{\IEEEauthorblockA{\IEEEauthorrefmark{2}Corresponding Author: Yuan Fang.}}
}

\author{Zhihao Wen, Yuan Fang,~\IEEEmembership{Senior Member,~IEEE}, Pengcheng Wei, Fayao Liu, Zhenghua Chen,~\IEEEmembership{Senior Member,~IEEE}, Min Wu,~\IEEEmembership{Senior Member,~IEEE}
\thanks{Zhihao Wen and Yuan Fang are with School of Computing and Information Systems, Singapore Management University, Singapore, 188065. (Email: zhwen.2019@phdcs.smu.edu.sg, yfang@smu.edu.dg)}
\thanks{Pengcheng Wei is with Information Systems Technology and Design Pillar, Singapore University of Technology and Design, Singapore. (Email: pengcheng\_wei@mymail.sutd.edu.sg)}
\thanks{Fayao Liu, Zhenghua Chen and Min Wu are with Institute for Infocomm Research, A*STAR, Singapore. (Email: liu\_fayao@i2r.a-star.edu.sg, chen0832@e.ntu.edu.sg, wumin@i2r.a-star.edu.sg)}
\thanks{Corresponding author: Yuan Fang.}}

\maketitle

\begin{abstract}
Predicting Remaining Useful Life (RUL) plays a crucial role in the prognostics and health management of industrial systems that involve a variety of interrelated sensors. Given a constant stream of time series sensory data from such systems, deep learning models have risen to prominence at identifying complex, nonlinear temporal dependencies in these data. 
% Derived from various sensors, spatial dependencies manifest as sensor correlations. Current approaches face limitations in adequately modeling and capturing these spatial dependencies, hindering their ability to learn representative features crucial for RUL prediction. To address these challenges, we introduce a novel
In addition to the temporal dependencies of individual sensors, spatial dependencies emerge as important correlations among these sensors, which can be naturally modelled by a temporal graph that describes time-varying spatial relationships. However, the majority of existing studies have relied on capturing discrete snapshots of this temporal graph, a coarse-grained approach that leads to loss of temporal information. Moreover, 
given the variety of heterogeneous sensors, it becomes vital that such inherent heterogeneity is leveraged for RUL prediction in temporal sensor graphs. 
To capture the nuances of the temporal and spatial relationships and heterogeneous characteristics in an interconnected graph of sensors, we introduce a novel model named Temporal and Heterogeneous Graph Neural Networks (THGNN). Specifically, THGNN aggregates historical data from neighboring nodes to accurately capture the temporal dynamics and spatial correlations within the stream of sensor data in a fine-grained manner. Moreover, the model leverages Feature-wise Linear Modulation (FiLM) to address the diversity of sensor types, significantly improving the model’s capacity to learn the heterogeneity in the data sources. Finally, we have validated the effectiveness of our approach through comprehensive experiments. Our empirical findings demonstrate significant advancements on the N-CMAPSS dataset, achieving
% with a 10.6\% improvement in RMSE and a 20.4\% increase in Score.
improvements of up to 19.2\% and 31.6\% in terms of two different evaluation metrics over state-of-the-art methods.
\end{abstract}


\section{Introduction}\label{sec:intro}

\lettrine[lines=2,slope=-4pt,nindent=-4pt]{P}{redicting} the Remaining Useful Life (RUL) of modern machinery is a pivotal aspect of machine health monitoring \cite{ren2022mctan, gao2020neural, li2021degradation}, which is essential for reducing maintenance costs, preventing unexpected failures, and enhancing system reliability. Deep Learning (DL) techniques have emerged as the state of the art in the RUL prediction field \cite{wang2020remaining}, given their superior ability to extract powerful features from complex data, thereby establishing new performance standards. These methods primarily employ temporal encoders, such as Convolutional Neural Networks (CNNs) \cite{li2021degradation, sateesh2016deep, wen2019new} and Long Short-Term Memory (LSTM) networks \cite{fu2021novel, zhang2018long, chen2020machine, xu2021kdnet}, which are adept at detecting temporal patterns in time-series data, a critical factor in accurate RUL prediction. Typically, RUL prediction often relies on Multivariate Time-Series (MTS) data collected from various sensors. These sensors monitor different machine aspects, such as temperature, pressure, and fan speed. The sensor measurements reveal significant interactive dynamics that often arise from the spatial arrangement of the sensors. For instance, an increase in fan speed is frequently associated with a rise in temperature in the vicinity, highlighting complex interrelations or interdependencies between these sensor measurements. Recognizing these interdependencies \cite{wang2023local} is crucial for improving the accuracy of RUL prediction.

Despite this, many studies have overlooked these spatial correlations, which has hindered the effectiveness of RUL prediction. A handful of recent studies have started to explore the importance of spatial dependencies, attempting to include them by using 2D-CNNs \cite{zhao2020double} or by conduction preliminary analysis of sensor correlations \cite{hong2020multivariate}. However, a gap remains in fully leveraging this intricate spatial information, which in turn limits the performance of RUL prediction models. To bridge this gap, some researchers have turned to constructing graphs to capture the full spatial relationships among the sensors \cite{wang2023local}.
Specifically, various components within the machine, such as the sensors in an aircraft engine (see an example illustration in Figure 2 of a previous work \cite{arias2021aircraft}), can be conceptualized as nodes in a graph. The spatial relationships among them naturally constitute the edges within this graph.
Given this sensor graph, they futher leverage Graph Neural Networks (GNNs) to effectively capture various spatial dependencies.

\stitle{Challenges and Our Approach.}
While we can formulate the spatial relationships among sensors into a graph, addressing the spatial dependencies for  RUL prediction still presents two significant challenges.

\textit{Challenge 1: Utilizing fine-grained temporal information in graph structures for RUL prediction.}
The sensor graph is inherently temporal, where a continuous steam of values is captured for each sensor over time.
% Gradual degradation in industrial systems often manifests in subtle temporal patterns that snapshot-based methods fail to capture, resulting in prediction inaccuracies.
While gradual degradation in industrial systems often manifests in subtle temporal patterns, prior research has often segmented temporal graphs into static snapshots \cite{du2018dynamic, li2018deep, goyal2018dyngem, zhou2018dynamic, sankar2020dysat, pareja2020evolvegcn}, a method that simplifies the modeling process but fails to capture the fine-grained evolving dynamics over time due to the inevitable loss of temporal information within each snapshot, as shown in Figure~\ref{fig:segmentation} (a).
Conversely, continuous-time approaches \cite{nguyen2018continuous, xu2020inductive} struggle to model complex interactions and how past events influence future occurrences throughout the  graph. Drawing inspirations from the Hawkes process \cite{hawkes1971spectra, mei2017neural,wen2022trend}, which highlights how past events can stimulate or ``excite" future occurrences, our model focuses on aggregating comprehensive historical information rather than relying on simplified historical graph snapshots. This approach enables the modelling of complex temporal dependencies in a fine-grained manner.

\textit{Challenge 2: Addressing heterogeneity among diverse types of sensor for RUL prediction.}
Our constructed sensor graph is not only temporal, but also heterogeneous, encompassing a variety of sensor types as illustrated in Table~\ref{tab:heterogeneity}.  Hence, our sensor graph includes a diverse array of node and edge types, forming what is known as a \emph{heterogeneous information network} or \emph{heterogeneous graph}  \cite{wang2022survey}. While heterogeneous graphs are widely used in various data mining tasks due to their rich semantic content, they have not been extensively studied in the temporally evolving setting. In particular, previous studies on RUL prediction \cite{wang2023local} have not adequately considered the inherent heterogeneity of various machine components and sensor measurements in the context of a spatially structured graph. To address the heterogeneity, it is important to strike a balance between preserving the unique individual characteristics of each sensor type and modeling the global patterns of heterogeneous information from different sensor types. On one hand, fitting one model for each sensor type  would overlook the global patterns; on the other hand, fitting a single global model would treat all types uniformly and miss out on the distinct features of each sensor type. To this end, we employ Feature-wise Linear Modulation (FiLM) \cite{perez2018film}, which enables individual adaptation to each sensor type from an underlying global model that aims to capture higher-level patterns across different types.
%effectively manage sensor heterogeneity, thereby enhancing the model’s learning capabilities from diverse data.

\stitle{Contributions.}
To address the above challenges, we proposed a novel approach called \emph{Temporal and Heterogeneous Graph Neural Networks}
(THGNN) for the problem of RUL prediction. THGNN employs a graph neural network backbone, aiming to model the spatial relations between sensors in a continuously timed temporal graph, as well as the diverse characteristics of various sensor types in a heterogeneous graph. 
In summary, this work makes the following notable contributions.
\begin{itemize}
    \item We identify the importance of capturing both the fine-grained temporal dynamics and diverse semantics of heterogeneous sensor types in RUL prediction.
    \item We model the RUL prediction problem via a continuous temporal and heterogeneous graph, proposing a novel model called THGNN. On the temporal side, to enable historical events to excite future events, THGNN aggregates historical information from neighboring sensor nodes, thereby capturing continuous temporal dynamics and spatial correlations in one model. On the heterogeneous side, we adapt to individual sensor types from a common model, allowing for a balance between the individual characteristics of each sensor type and the global patterns across different types. 
    \item We conduct extensive experiments on the C-MAPSS and N-CMAPSS benchmarks to demonstrate the effectiveness of our approach on the real-world engine data. Notably, on the newer N-CMAPSS dataset, which presumably mirrors real-world scenarios more closely, our method THGNN has achieved a 19.2\% improvement in RMSE and a 31.6\% boost in Score over state-of-the-art baselines, highlighting the practical relevance and effectiveness of our proposed solution.
\end{itemize}

\begin{figure}[t]
   \centering
   \includegraphics[width=0.8\linewidth]{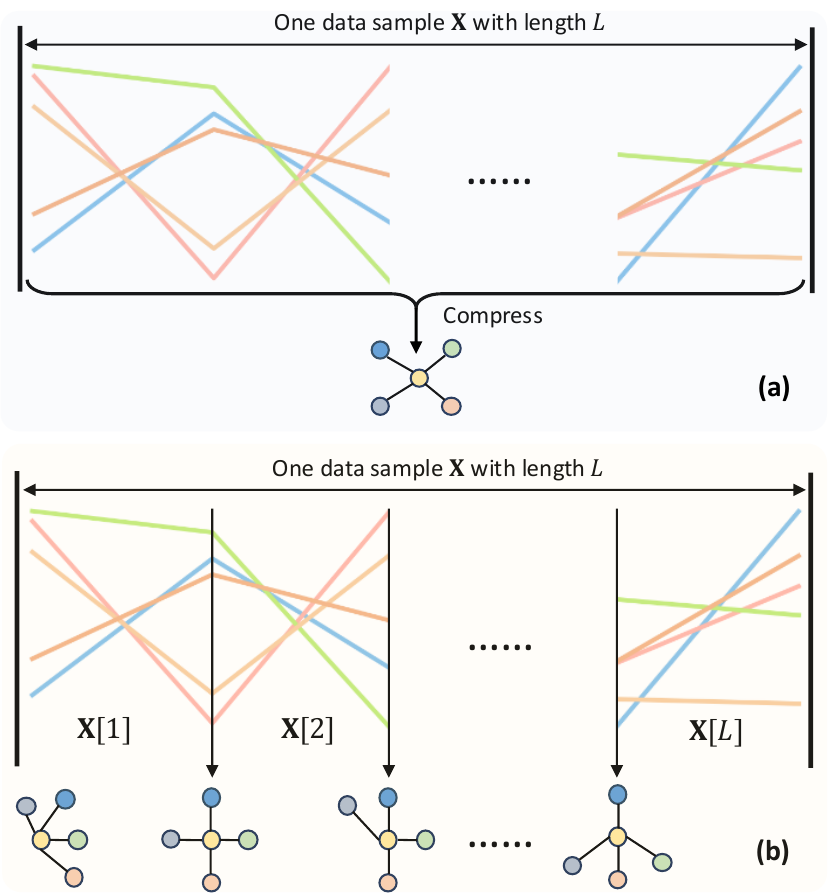}
   % \vspace{-2mm}
	\caption{
 Illustration of one MTS data sample $\vec{X}$ across $L$ timestamps.
 % we have harnessed the information at every minimal time step. 
 % Previous work \cite{wang2023local} formulates MTS data into graph snapshots, where each snapshot represents a condensed version of multiple timestamps over a certain period, leading to the loss of fine-grained temporal information within each snapshot. In contrast, our approach retains the full temporal information by considering \textbf{every minimal time step, taking each timestamp as one distinct step}}.
 Previous snapshot-based methods \cite{wang2023local, wang2021spatio} typically \textbf{compress} MTS data from a time window into \textbf{a single ``flat" graph} (or several graphs but at a coarser granularity than every step), often leading to the \textbf{loss of critical fine-grained temporal details}, as shown in (a). 
 In contrast, our method constructs a graph \textbf{at each time step} and performs \textbf{fine-grained temporal aggregation}, accounting for historical dependencies across \textbf{every time step}.
 This ensures that our model fully leverages all available temporal information without compressing or discarding details.
 }
 
	% \vspace{-2mm}
	\label{fig:segmentation}
\end{figure}

\begin{table}[t]
  \centering
  \small
  % \centering
  % \footnotesize 
  %\addtolength{\tabcolsep}{-3pt}
  \caption{\textbf{Heterogeneity of the sensors}.
  An engine involves heterogeneous types of sensor, including those that measure wind speed, temperature, and other parameters \cite{frederick2007user,arias2021aircraft}.
  Not only do these sensors exhibit significant differences in the numerical values of their measurements, but the magnitude and frequency of changes in their measurements also vary widely. ``Type id'' is a unique identifier we have assigned to each type of sensor. 
  In addition, different sensors have different temporal behaviors.
  }
  \begin{tabular}{c|l|c|c}
\toprule
\textbf{Sensor} & \textbf{Description} & \textbf{Unit} & \textbf{Type id} \\
\midrule
T24 & Total temperature at LPC outlet & $^\circ $R & 1 \\
T30 & Total temperature at HPC outlet & $^\circ $R & 1 \\
T50 & Total temperature at LPT outlet & $^\circ $R & 1 \\
P30 & Total pressure at HPC outlet & psia & 2 \\
Ps30 & Static pressure at HPC outlet & psia & 2 \\
Nf & Physical fan speed & rpm & 3 \\
Nc & Physical core speed & rpm & 3 \\
Phi & Ratio of fuel flow to Ps30 & pps/psi & 4 \\
NRf & Corrected fan speed & rpm & 5 \\
NRc & Corrected core speed & rpm & 5 \\
BPR & Bypass ratio & -- & 6 \\
htBleed & Bleed enthalpy & -- & 6 \\
W31 & HPT coolant bleed & lbm/s & 7 \\
W32 & LPT coolant bleed & lbm/s & 7 \\
\bottomrule
\end{tabular}
  \label{tab:heterogeneity}
\end{table}

\section{Related work}\label{sec:related work}

In this section, we present an overview of the literature for RUL prediction in two broad categories: conventional deep learning-based approaches and graph-based approaches.

\subsection{Conventional Deep Models for RUL Prediction}

Deep learning (DL) models, established for their robust feature-learning capabilities, have emerged as a promising direction for RUL prediction. The adoption of convolutional neural networks (CNNs), initially spurred by their success in computer vision, has been extended to address industrial problems like fault diagnosis \cite{song2021early, wen2020new} and surface defect recognition \cite{wang2020remaining}. For RUL prediction, early efforts have focused on 1D-CNN applications \cite{sateesh2016deep,wen2019new}. 
%, with Baru et al.~\cite{sateesh2016deep} demonstrating a 1D-CNN framework that has outperformed shallow learning models. Wen et al.~\cite{wen2019new} has further introduced a k-fold ensemble-based residual 1D-CNN model, enhancing RUL prediction accuracy.
Meanwhile, recurrent neural networks have also played a crucial role in advancing RUL prediction. Numerous studies have explored the long-short term memory (LSTM) and its variants, with notable contributions like the bi-directional LSTM model \cite{zhang2018long} in identifying patterns in time-series data, surpassing other machine learning approaches. Later on, Chen et al.~\cite{chen2020machine} has integrated LSTM with attention mechanisms to highlight key features and time steps, while Xu et al.~\cite{xu2021kdnet} has integrated knowledge distillation with LSTM, yielding robust predictions.

Recently, Transformer-based models \cite{vaswani2017attention}, known for their powerful pairwise attention mechanism, have gained traction, inspiring extensive research to leverage their temporal dependency modeling in time-series forecasting tasks \cite{zerveas2021transformer, wu2021autoformer, zhou2021informer}. These novel approaches have significantly propelled the field of RUL prediction forward.

However, existing research with a predominant focus on temporal dependencies has overlooked the importance of spatial dependencies, i.e., sensor correlations, which are crucial for the comprehensive modeling of MTS data. Recent studies \cite{zhao2020double, hong2020multivariate, jiang2022adversarial, wang2021deep} have recognized this gap, attempting to capture spatial dependencies by integrating sensor correlations into RUL prediction models. Approaches employing 2D-CNN \cite{zhao2020double} or 3D-CNN \cite{jiang2022adversarial, wang2021deep} have been attempted, but these CNNs do not inherently discern sensor correlations, resulting in a gap in modeling these correlations fully and effectively. Other methods include using sensor correlations as a pre-processing step \cite{hong2020multivariate}, yet this overlooks their potential in feature learning for RUL prediction. Thus, there is a pressing need for methodologies that fully capture spatial dependencies in end-to-end learning  to enhance the accuracy of RUL prediction.

\subsection{Predicting RUL with Graph Neural Networks}

Graph neural networks (GNNs) have been designed to model correlations between entities \cite{kipf2016semi,zhao2019t,jia2021multi,liu2022dual}, which has catalyzed research in RUL prediction to capture sensor correlations. For instance, Wei et al.~\cite{wei2023bearing} have employed GNNs on graphs constructed from various timestamps to improve RUL prediction. Zhang et al.~\cite{zhang2021adaptive} have proposed an adaptive spatio-temporal GNN specifically tailored for RUL prediction, while Li et al.~\cite{li2021hierarchical} have developed a hierarchical attention-based GNN that efficiently synthesizes signals from multiple sensors. These GNN-based approaches have consistently outperformed conventional deep learning-based methods that do not account for sensor correlations.
Although spatiotemporal GNN methods \cite{Cini2023GraphDL} like Graph WaveNet \cite{Wu2019GraphWF} , AGCRN \cite{Bai2020AdaptiveGC}, GNNs for human pose prediction \cite{Fu2022LearningCD} and traffic prediction \cite{Zhao2018TGCNAT}, and heterogeneous GNNs \cite{Jin2023ASO, Deng2020DynamicKG},  have been proposed, they have not been applied to MTS RUL prediction.

% However, one critical limitation of these GNN-based methods is that they segment temporal graphs into discrete, static snapshots \cite{du2018dynamic, li2018deep, goyal2018dyngem, zhou2018dynamic, sankar2020dysat, pareja2020evolvegcn}. The discretization simplifies the modeling process but fails to capture the evolving temporal dynamics within each snapshot.

However, one critical limitation of many GNN-based methods for RUL prediction is their segmentation of temporal graphs into discrete, static snapshots \cite{wang2023local, wei2023bearing, zhang2021adaptive, li2021hierarchical, wang2021spatio}. This simplification leads to the loss of fine-grained temporal dynamics, which are critical for accurately capturing the degradation patterns in machinery.
Another significant drawback is that they overlook the heterogeneity \cite{wang2022survey} of the various sensors. In a modern system, there often exist a wide variety of sensors, which measure different physical quantities and exhibit significant variations in their values and trends.

\section{Preliminaries}

In this section, we formalize the problem statement, sensor heterogeneity, and present an introduction on graph neural networks as the backbone architecture.

\subsection{Problem definition}
% \wen{need rewrite to emphasize the difference of our model and snapshot based model,  added a discussion in Section III highlighting the relationship between our approach and attention mechanisms}
Consider a Multivariate Time-Series (MTS) dataset, which is composed of \( K \) labeled samples \( \{\vec{X}_{i}, y_{i}\}_{i=1}^{K} \). Each MTS sample \( \vec{X}_{i} \in \mathbb{R}^{L \times N} \) is derived from \( N \) sensors across \( L \) timestamps, as illustrated in Figure~\ref{fig:segmentation}. The label of each sample, $y_i\in \mathbb{R}$, is a numerical value representing the remaining useful life of the machine at the point of capturing the sample $\vec{X}_{i}$. The aim is to develop a GNN-based regression model to predict the RUL of each sample, focusing on capturing  the spatial dependencies of the sensors, as well as the dynamic and heterogeneous nature of the MTS data. To simplify the notation, we omit the subscript \( i \) when there is no ambiguity, thus denoting each sample as \( \vec{X} \).
And as shown in Figure~\ref{fig:segmentation} (b),  \( \vec{X}[L] \) refers to the temporal features of sensors at the $L$-th time step, used for $L$-th graph construction.

\subsection{Sensor heterogeneity}
We define \emph{sensor heterogeneity} as the presence of diverse sensor types and their distinct temporal behaviors within a multivariate time-series system. Formally, let $\mathcal{S} = \{s_1, s_2, \ldots, s_N\}$ denote the set of $N$ sensors, where each sensor $s_i$ is associated with:
\begin{itemize}
    \item a semantic type identifier $t_i \in \mathcal{T}$ (e.g., pressure, temperature, speed), and
    % \item a temporal signal $\mathbf{x}_i = \{x_i^1, x_i^2, \ldots, x_i^L\}$ over $L$ time steps.
    \item a temporal signal $\mathbf{x}_i = \{x_i^1, x_i^2, \ldots, x_i^L\}$ over $L$ time steps, where each $x_i^t \in \mathbb{R}$ denotes the scalar reading of sensor $s_i$ at time $t$.
\end{itemize}
We formalize sensor heterogeneity as a joint distribution:
\begin{align}
    \mathcal{H} = P(t_i, \mathbf{x}_i), \quad i = 1, \ldots, N,
\end{align}
which captures both the semantic (type-level) and dynamic (time-dependent) variability among sensors.

This definition serves as the basis for the heterogeneity-aware representations and interactions introduced in subsequent sections.

\subsection{Graph neural networks}
Our methodology is based on graph neural networks (GNNs), which possess an inductive capacity to handle new samples (graphs) assuming a shared feature space across existing and new nodes. Hence, they are well suited for MTS data, given common multivariate dimensions across different samples and timestamps. Next, we provide a concise overview of GNNs in the following.

Contemporary GNNs predominantly follow a message-passing framework. In this framework, each node in a graph receives, aggregates and update messages (i.e., features) from its adjacent nodes, enabling a process of messages passing along the graph structures. The message-passing layer can be stacked multiple times to broaden the receptive field of each node. More specifically, at each layer, the process can be mathematically represented as below.
\begin{align}
    \vec{h}^{l+1}_v = \bM\left(\vec{h}^{l}_v, \left\{\vec{h}^{l}_u\mid \forall u \in \bN_v\right\}; \vec{W}^l\right),
\end{align}
where $\vec{h}^l_v$ denotes the $l$-th layer's message (which is a $d_l$-dimensional embedding vector) for node $v$, $\bN_v$ represents the set of neighboring nodes of $v$, $\vec{W}^l$ is a learnable weight matrix  that transforms the node embeddings from the $l$-th layer, and $\bM(\cdot)$ is a function that aggregates the messages.
The initial message for node $v$ at the input layer is essentially the node's original features, denoted as $\vec{h}^1_v\equiv\vec{x}_v$.

\begin{figure*}
   \centering
   \includegraphics[width=1\linewidth]{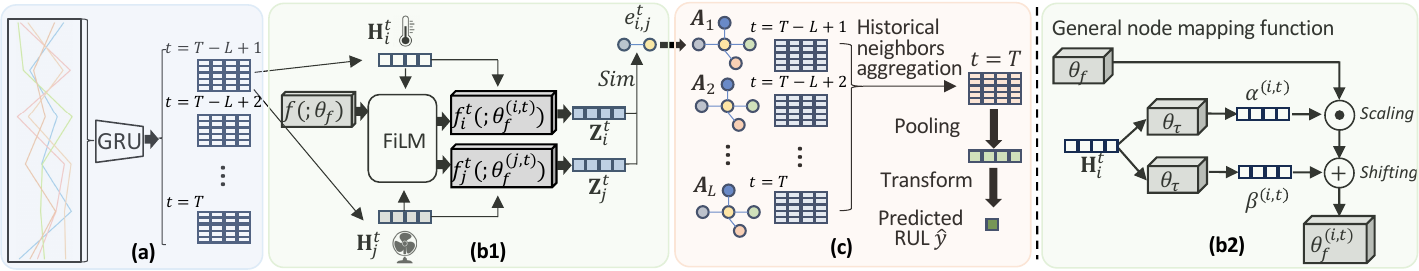}
   % \vspace{-2mm}
	\caption{Overall framework of \model.
 \textbf{(a)} We input a sample into a GRU, yielding $L$ embedding matrices \(\{\vec{H}_t\mid t\}\) for each time step $t$. \textbf{(b1)} For a given time step $t$, we construct edges by calculating the similarity between each pair of sensors, where \textbf{(b2)} describes the details of the FiLM layer in (b1). Consequently, we obtain a temporal graph across the $L$ time steps. 
  %the latent embeddings \(\vec{Z}^t_i\) and \(\vec{Z}^t_j\) of sensors $i$ and $j$. 
 %Consequently, we obtain the graph for time $t$. Similarly, we generate $L$ graphs for the time steps from $t=T-L+1$ to $t=T$. 
 \textbf{(c)} We perform message passing on the temporal graph at each time step, feeding in the corresponding embedding matrix \(\vec{H}_t\) at step $t$, and obtain $L$ latent embedding matrices across the time steps. The $L$ embedding matrices are further aggregated and processed to obtain the final representation at the target time step $T$. Lastly, we predict the RUL at $t=T$ based on the final representation. 
 }
	\vspace{-2mm}
	\label{fig:framework}
\end{figure*}

\section{Proposed approach}

In this section, we introduce our approach called Temporal and Heterogeneous Graph Neural Networks (\model). The overall structure of \model\ is shown in Fig.~\ref{fig:framework}, given MTS data as input, 
we employ a more granular approach to graph construction while also taking into account the heterogeneity of sensors.

\subsection{Sensor Temporal Representation}
As depicted in Fig~\ref{fig:framework}(a), our input consists of MTS data. Conventionally, recurrent neural networks (RNNs) are used directly to exploit temporal data sequentially for end-to-end prediction. However, this approach has a significant limitation: It fails to capture and utilize the structural information in the lateral graph that exists between the multivariate dimensions. Thus, our methodology focuses on leveraging the strengths of RNNs in processing temporal information to obtain an initial representation of the MTS data. 

Formally, we have a sample \(\vec{X} \in \mathbb{R}^{L \times N}\), which originates from \(N\) sensors on \(L\) timestamps. We then feed \(\vec{X}\) into a GRU \cite{cho2014learning}, a straightforward yet effective recurrent neural network architecture, which outputs \(L\) temporal representation matrices, as follows.
\begin{align}
    \vec{H}_{T-L+1},\vec{H}_{T-L+2}, \ldots, \vec{H}_{T} = \textsc{GRU}(\vec{X};\theta_{r}), 
    \label{eq:temporal representation matrices}
\end{align}
where \(\vec{H}_{t} \in \mathbb{R}^{N \times d}\) is the representation for time step $t \in\{T-L+1,T-L+2,\ldots,T\}$, \(N\) represents the number of sensors, \(d\) is the hidden dimension, and $\theta_{r}$ is the parameter set of the GRU. 
% This matrix encapsulates the representation of each sensor at time \(t\).
We choose GRU instead of the popular self-attention layer, because GRU is effective and robust, particularly on datasets with limited training samples.

\subsection{Constructing a Temporal Heterogeneous Sensor Graph}

In this subsection, we aim to construct a graph to capture the correlations between heterogeneous sensor signals within and across time steps.

\stitle{Heterogeneous Sensor Graph.}
In a complex system, sensors exhibit spatial relationships due to the intricate arrangement of the sensors in the system. Hence, within each time step, the readings from different sensors are often correlated given their spatial relationships. 
While some prior works in spatio-temporal GNNs have used physical sensor locations to define graph links, this approach often proves \textbf{suboptimal for RUL prediction} \cite{wang2021spatio}. Sensors that are spatially adjacent may differ significantly in behavior, while those that are spatially distant may exhibit strong correlations. Therefore, like many GNN-based RUL prediction methods \cite{wang2023local}, we opted to use cosine similarity to construct the graph, ensuring that it reflects the true functional relationships between sensors.
Regarding the use of fully connected graphs, while such graphs can be inefficient in scenarios with many nodes, our task involves relatively \textbf{small graphs} due to the modest number of sensors in typical industrial systems. Consequently, the computational overhead is negligible, and standard resources can handle the fully connected graph efficiently. This approach also ensures that \textbf{no potentially relevant sensor relationships are omitted}.
Specifically, the cosine similarity between the representations of the two sensor nodes at time $t$, \( \vec{H}^{t}_{i} \) and \( \vec{H}^{t}_{j} \) is:
\begin{align}
    \vec{A}_t(i,j) = \langle \vec{H}^{t}_{i}, \vec{H}^{t}_{j} \rangle,
\end{align}
where $\vec{A}_t(i,j) \in [0, 1]$ denotes the edge weight between $i$ and $j$, and $\langle \cdot \rangle$ denotes the normalized cosine similarity.

%need change the Figure into Table
However, this method overlooks the heterogeneity of sensors in a modern machine. In a complex system, \textbf{various sensors often belong to multiple types, measuring heterogeneous  physical quantities (e.g., fan speed, temperature, or pressure)}, as illustrated in Table~\ref{tab:heterogeneity}. On one hand, the heterogeneity adds to the complexity of the local correlations among sensors. Not only do these sensors differ significantly in their numerical readings, but their fluctuations in the readings also exhibit distinct patterns. In other words, the representations between different types of sensors show considerable differences, and directly computing their similarities without considering these disparities is problematic. On the other hand, the heterogeneity also provide additional insights for comprehensively capturing the spatial dependencies among sensors. 

While previous works have utilized graph structures for machine RUL prediction \cite{zhang2021adaptive, li2021hierarchical, wei2023bearing}, none of them has accounted for the heterogeneity of sensors. Simultaneously, although there are numerous studies on heterogeneous graphs \cite{wang2022survey}, they have not been applied to the task of RUL prediction on MTS data.  
To encapsulate the heterogeneity of sensors into RUL prediction, we first discern the edges in the sensor graph.
A classical approach involves the use of meta-paths \cite{fu2020magnn}, 
%which refer to sequences of node and edge types connecting two nodes. This concept 
which serve as high-level abstractions for defining semantic relationships between nodes in a heterogeneous graph. 
% For instance, in a citation network with nodes representing papers and authors, and edges indicating citations and collaborations, a meta-path might be "Paper - cites -\> Paper - written by -\> Author". This meta-path delineates a sequence of relationships between papers and authors, highlighting the citation and authorship connections.
However, meta-paths are often predefined based on domain knowledge, which has not been studied in sensor graphs. 
%in our engine graph, we lack predefined edges, preventing us from further defining our meta-paths.

Instead, we could establish a learnable mapping function $f_{\phi}(\cdot)$ for each node (sensor) type $\phi$ \cite{wang2019heterogeneous}, which maps the representations of different types of nodes to the same space:
\begin{align}
    \vec{A}_t(i,j) = \langle f_{\phi(i)}(\vec{H}^{t}_{i}), f_{\phi(j)}(\vec{H}^{t}_{j}) \rangle,
    \label{eq:hete functions}
\end{align}
where $\phi(i)$ is the type ID of sensor $i$. 
However, this approach requires learning one set of parameters for each sensor type, which would cause a substantial increase in the number of parameters for a large number of types. 
Furthermore, this approach only captures coarse-grained differences between types, failing to account for the fine-grained differences within each type. For instance, despite two sensors both measuring temperature, deployment in different physical environment can still cause significant differences in their readings. 
Extending this observation, even the same sensor could exhibit significant variations across time steps. Hence, in the following, we propose a universal mapping function that can be adapted to different nodes and time steps.

\stitle{Universal Adaptation across Nodes and Time.}
To develop a universal mapping function that is both parameter-efficient and capable of adapting to heterogeneous nodes and different time steps, we introduce a \emph{temporal sensor-conditioned mapping} approach. 

Consider a global mapping function $f(\cdot;\theta_f)$ to map node embeddings into a unified space.
We adopt a multi-layer perceptron (MLP) with a two-layer bottleneck structure \cite{zhou2022conditional} for the implementation of the global function $f(\cdot)$.
To adapt the global function to different nodes and time steps, we transform the global parameters $\theta_f$ into sensor and time-specific parameters $\theta^{(i,t)}_f$, for any sensor $i$ and time $t$. That is, we derive a local mapping function for each sensor $i$ and time $t$, $f^t_i(\cdot;\theta^{(i,t)}_f)$, where the local parameters $\theta^{(i,t)}_f$ are generated by a secondary function $\tau(\cdot)$, as follows.
\begin{align}
    \theta^{(i,t)}_f = \tau (\theta_f, \vec{H}_{i}^{t}; \theta_\tau),
    \label{eq:tau}
\end{align}
which is parameterized by $\theta_\tau$. Importantly, its input is conditioned on both the global parameters $\theta_f$, and the latent representation $\vec{H}_{i}^{t}$ for sensor $i$ at time $t$, thereby adapting the global parameters to this specific sensor and time.

This approach can be viewed as a form of hypernetwork \cite{DBLP:conf/iclr/HaDL17}, where a secondary neural network ($\tau$) generates the local parameters for the primary network ($f$). The advantage of our temporal sensor-conditioned approach is threefold. First, it is parameter-efficient as only the global parameters, $\theta_f$, and the parameters of the secondary network, $\theta_\tau$, are learnable. Second, the global parameters can capture the common patterns across the sensors and time. Third, the secondary network enables quick adaptation to each sensor and time step, thereby capturing their unique characteristics. 

\stitle{Materialization of $\tau$.}
We employ \emph{Feature-wise Linear Modulation} (FiLM) \cite{perez2018film, wen2022trend, wen2021meta}, which utilize affine transformations including scaling and shifting on the global parameters $\theta_{f}$, the parameters of the global mapping function $f(\cdot)$, conditioned on the temporal node representation $\vec{H}^{t}_{i}$, as shown in Figure~\ref{fig:framework}(b2).
Compared with gating \cite{xu2015show}, FiLM is more flexible in adjusting the parameters and can be conditioned on arbitrary input, as shown below.
\begin{align}
\theta_{f}^{(i,t)}=\tau (\theta_{f}, \vec{H}^{t}_{i}; \theta_\tau)=(\alpha^{(i,t)} + \vec{1})\odot \theta_{f} + \beta^{(i,t)},
\label{eq:learnable transform}
\end{align}
where $\odot$ stands for element-wise multiplication, $\alpha^{(i,t)}$ and $\beta^{(i,t)}$ perform element-wise scaling and shifting, respectively, and $\vec{1}$ is a vector of ones to ensure that the scaling factors are centered around one.
Note that $\theta_{f}$ contains all the weights and biases of $f(\cdot)$, and we flatten it into a $d_{f}$-dimensional vector. Hence, $\alpha^{(i,t)}$ and $\beta^{(i,t)}$ are $d_f$-dimensional vectors, too.

Specifically, we employ fully connected layers to generate the scaling vectors $\alpha^{(i,t)}$ and shifting vectors $\beta^{(i,t)}$, conditioned on the temporal node representation $\vec{H}^{t}_{i}$, as follows. 
\begin{align}
\alpha^{(i,t)}=\sigma(\vec{H}^{t}_{i} \vec{W}_{\alpha} + \vec{b}_{\alpha})
,\\
\beta^{(i,t)}=\sigma(\vec{H}^{t}_{i}\vec{W}_{\beta} + \vec{b}_{\beta}), 
\end{align} 
where $\vec{W}_{\alpha}, \vec{W}_{\beta} \in \mathbb{R}^{d \times  d_{f}}$ and $\vec{b}_{\alpha},\vec{b}_{\beta} \in \mathbb{R}^{d_{f}}$ are learnable weight matrices and bias vectors of the fully connected layers.
%, in which $d$ is the dimension of node representation and $d_{f}$ is the total number of parameters in the general node mapping function $f(\cdot)$. 
%The output $\alpha^{(i,t)},\beta^{(i,t)} \in \mathbb{R}^{d_{f}}$ are both $d_{f}$-dimensional vectors, which represent the scaling and shifting operations of the transformation model $\tau$ in Eq.~\eqref{eq:tau}. They are used to transform the \emph{parameters} $\theta_{f}$ of the general node mapping function $f(\cdot)$ into temporal node $(i,t)$-specific parameters by element-wise scaling and shifting, given by 

In summary, the learnable transformation model $\tau$ is parameterized by $\theta_\tau=\{\vec{W}_{\alpha}, \vec{b}_{\alpha}, \vec{W}_{\beta}, \vec{b}_{\beta}\}$, \ie, the collection of parameters of the fully connected layers that generate the scaling and shifting vectors. 
Furthermore, $\tau$ is also a function of the temporal node condition $\vec{H}^{t}_{i}$, for $\alpha^{(i,t)}$ and $\beta^{(i,t)}$ are functions of the temporal node condition.

\stitle{Temporal Heterogeneous Sensor Graph Construction.}
% After finishing the temporal node-conditioned transformation, we have 
% as shown in Figure~\ref{fig:framework} (b), 
The FilM layer materializes the local mapping functions $f^{t}_{i}(\cdot;\theta_{f}^{(i,t)})$ for sensor $i$ and time $t$. Given $\vec{H}^{t}_{i}$ and $\vec{H}^{t}_{j}$ as input to the FiLM, as shown in Figure~\ref{fig:framework}(b1),  we can map different nodes at different time steps into a unified space, and obtain the following transformed representation $\vec{Z}^{t}_{i} \in \mathbb{R}^{d}$:
%using the node mapping function, resulting in $\vec{Z}^{t}_{i}$ and $\vec{Z}^{t}_{j}$, as demonstrated by the following equation:
\begin{align}
    \vec{Z}^{t}_{i} = f^{t}_{i}(\vec{H}^{t}_{i};\theta_{f}^{(i,t)}).
   % \vec{Z}^{t}_{i}, \vec{Z}^{t}_{j} = f^{t}_{i}(\vec{H}^{t}_{i};\theta_{f}^{(i,t)}), f^{t}_{j}(\vec{H}^{t}_{j};\theta_{f}^{(j,t)}).
\end{align}
Next, we define the edge weight between sensor $i$ and node $j$ at time $t$ by calculating the cosine similarity of $\vec{Z}^{t}_{i}$ and $\vec{Z}^{t}_{j}$:
\begin{align}
    \vec{A}_t(i,j) = \langle \vec{Z}^{t}_{i}, \vec{Z}^{t}_{j} \rangle.
\end{align}

\subsection{Temporal Message Passing}

The temporal heterogeneous sensor graph, 
%Similarly, other edges in the graph at time $t$ are established in this manner. Once all edges are established, the graph at time $t$ is completely constructed. Following this procedure, 
as shown in Figure~\ref{fig:framework}(c), consists of a series of $L$ graph structures across the time steps $\{T-L+1,T-L+2,\ldots,T\}$. Hence, the full graph is characterized by the adjacency matrices at each time step, as follows. 
\begin{align}
    [\vec{A}_{1}, \vec{A}_{2}, \cdots, \vec{A}_{L}], 
\end{align}
where $A_1$ corresponds to $t=T-L+1$ and $A_L$ corresponds to $t=T$.
% Given the \(L\) temporal representation matrices in Eq.~\ref{eq:temporal representation matrices}, 

% \begin{align}
%     \vec{M}_{1}, \cdots, \vec{M}_{L} = \vec{A}_{1}\vec{H}_{T-L+1}\vec{W}_{nei}, \cdots, \vec{A}_{L}\vec{H}_{T}\vec{W}_{nei}, 
% \end{align}
At each time step, we 
%Building upon the previously obtained \(L\) temporal representation matrices, we can 
employ a straightforward message-passing GNN layer to transform and aggregate information from neighbors. This process results in \(L\) representation matrices that incorporate historical neighbor information, $\{\vec{M}_{s}\mid s=1,\ldots,L\}$, such that
\begin{align}
    %\vec{M}_{1}, \cdots, \vec{M}_{L} = \vec{A}_{1}\vec{H}_{T-L+1}\vec{W}_{nei}, \cdots, \vec{A}_{L}\vec{H}_{T}\vec{W}_{nei},
    \vec{M}_{s} = \vec{A}_{s}\vec{H}_{T-L+s}\vec{W}_\text{nei},
\end{align}
where $\vec{M}_{s} \in \mathbb{R}^{N \times d}$, 
%for $s \in \{1, 2, \cdots, L\}$, 
and $\vec{W}_\text{nei} \in \mathbb{R}^{d \times d}$ represents a learnable matrix shared among all time steps. $\{\vec{M}_{s}\mid s=1,\ldots,L\}$ will undergo further temporal aggregation to leverage historical information.

% \stitle{Interaction Between Heterogeneity and Temporal Aggregation.}
% Let $\mathbf{H}_t \in \mathbb{R}^{N \times d}$ denote the GRU-encoded representations of $N$ sensors at time $t$, and let $\mathbf{Z}_t = f_t(\mathbf{H}_t) \in \mathbb{R}^{N \times d}$ represent the FiLM-adapted embeddings that incorporate sensor heterogeneity. The temporal graph at time $t$ is constructed via cosine similarity between these embeddings:
% \[
% A_t(i,j) = \cos\left(Z_t^{(i)}, Z_t^{(j)}\right),
% \]
% where $Z_t^{(i)}$ is the embedding of sensor $i$ at time $t$.

% Message passing on the graph yields:
% \[
% \mathbf{M}_t = \mathbf{A}_t \mathbf{H}_t \mathbf{W}_{\text{nei}}.
% \]
% To aggregate across $L$ historical steps, we use a learnable exponential decay inspired by the Hawkes process:
% \[
% \mathbf{H}_{\text{nei}} = \sum_{s=1}^{L} a_s \mathbf{M}_s, \quad \text{where} \quad
% a_s = \frac{\exp\left(-\mu (L-s)\right)}{\sum_{j=1}^{L} \exp\left(-\mu (L-j)\right)}.
% \]

% Since each $\mathbf{A}_t$ is constructed based on FiLM-adjusted representations $\mathbf{Z}_t$, the temporal aggregation is implicitly conditioned on sensor heterogeneity. Specifically, the similarity structure $\mathbf{A}_t$ reflects both temporal context and semantic variations across sensor types, thus enabling heterogeneity-aware temporal modeling. This integration enhances the model’s ability to emphasize recent and semantically relevant information, improving prediction performance on complex, real-world systems.

\stitle{Temporal Aggregation of Historical Neighbors.}
Inspired by the concept of the Hawkes process \cite{zuo2018embedding}, we aim to aggregate all historical neighbor information gathered across the past $L$ time steps, 
i.e., $\{\vec{M}_{s}\mid s=1,\ldots,L\}$, in a time-decaying manner.
Intuitively, the excitement of events decreases with their temporal distance from the present; events closer in time can excite current events to a greater extent. Specifically, the aggregate influence from historical neighbors is given by
\begin{align}
 \vec{H}_\text{nei} &= \textstyle \sum_{s=1}^{L} a_{s} \vec{M}_{s}, 
\end{align}
where 
\begin{align}
a_{s} &= \textstyle\frac{\exp(-\mu (L-s))}{\sum_{j=1}^{L} \exp(-\mu (L-j))}, 
\end{align}
in which $\mu \in \mathbb{R}$ is a learnable scalar. 

The final representation at the target time $T$ further incorporates $\vec{H}_{T}$, the self-information of each node at time $T$, as follows.
\begin{align}
 \vec{H}_\text{self} &= \vec{H}_{T} \vec{W}_\text{self}, \\
 \vec{H}_\text{all} &= \sigma(\sigma( \vec{H}_\text{nei} +  \vec{H}_\text{self}) \vec{W}_\text{all}),
\end{align}
where $\vec{H}_\text{self} \in \mathbb{R}^{N \times d}$, \(\sigma\) represents the LeakyReLU activation function, \(\vec{W}_\text{all} \in \mathbb{R}^{d \times d}\) is a learnable weight matrix, and \(\vec{H}_\text{all} \in \mathbb{R}^{N \times d}\) is the final representation matrix.

\stitle{Overall Prediction and Objective.}
To recap, our objective is to predict the RUL of a system at the target time $T$, given a sample spanning $t=T-L+1,\ldots,T$. In other words, we aim to forecast a property of the system at a sample level. As a sample is represented by a temporal heterogeneous sensor graph, it is necessary to pool the final representation matrix to acquire a graph-level embedding, as follows.
\begin{align}
\vec{h} &= \textsc{POOL}(\vec{H}_\text{all}),
\end{align}
where $\vec{h} \in \mathbb{R}^{d}$ represents the graph-level embedding, and $\textsc{POOL}(\cdot)$ refers to a pooling operation which is implemented as a simple mean pooling in our experiments.

To predict the system's RUL, which is a scalar value representing the remaining useful life in some given unit, we further transform this graph-level embedding to produce a predictive value, $\hat{y}$, for the RUL:
\begin{align}
\hat{y} &= \sigma(\vec{h}^\top \vec{W}_\text{out}),
\end{align}
where $\hat{y} \in \mathbb{R}$ is our final prediction for the RUL at time $T$, and $\vec{W}_\text{out}\in \mathbb{R}^{d\times 1}$ is a learnable vector.

Subsequently, we jointly optimize all parameters $\Theta=(\theta_{r},\theta_{f},\theta_\tau, \vec{W}_\text{nei}, \vec{W}_\text{self}, \vec{W}_\text{all}, \vec{W}_\text{out})$, which include those of the GRU $\theta_{r}$, the global mapping function $\theta_{f}$, the FiLM adaptation layer $\theta_\tau$, the various weights in the temporal message-passing module, namely, $\vec{W}_\text{nei}, \vec{W}_\text{self}, \vec{W}_\text{all}$, and the final prediction weights for the output layer $\vec{W}_\text{out}$. More specifically, $\Theta$ is optimized w.r.t.~the following squared error:
%the neighborhood transformation weight matrix $\vec{W}_{nei}$, the self-information transformation weight matrix $\vec{W}_{self}$, the whole information transformation weight matrix $\vec{W}_{trans}$, and the final scalar transformation weight matrix $\vec{W}_{out}$, based on minimizing the following loss:
\begin{align}
\arg \min_{\Theta} \textstyle\sum_{i=1}^K (\hat{y}_{i} - y_{i})^2, 
\end{align}
where and $\hat{y}_{i}$ is the predicted RUL value of the $i$-th sample in the training set, and $y_{i}$ is the corresponding ground truth.

\subsection{Theoretical Insight on Heterogeneity-aware Temporal Aggregation}
\begin{theorem}[Heterogeneity-aware Temporal Aggregation]
Let $\vec{Z}_s = f_s(\vec{H}_s) \in \mathbb{R}^{N \times d}$ denote the FiLM-adapted sensor embeddings at time $s$, where $f_s(\cdot)$ is conditioned on both temporal features and sensor type. Let $\vec{A}_s(i,j) = \langle \vec{Z}^{s}_{i}, \vec{Z}^{s}_{j} \rangle$ be the similarity-induced adjacency matrix. Then the aggregated temporal neighbor representation:
\begin{equation}
\vec{H}_\text{nei} = \sum_{s=1}^{L} a_s \vec{A}_s \vec{H}_s \vec{W}_\text{nei}
\end{equation}
is jointly conditioned on the temporal decay weights $a_s$ and the sensor heterogeneity via $\vec{A}_s$, and cannot be decomposed into separable time-only or type-only aggregation components.
\end{theorem}

\begin{proof}
The weight $a_s$ encodes temporal influence as a scalar function of time difference $L-s$, while $\vec{A}_s$ is computed from the pairwise similarity of FiLM-modulated embeddings, i.e., $\vec{A}_s = \text{Sim}(f_s(\vec{H}_s))$, where $f_s(\cdot)$ incorporates sensor-type-aware modulation. Thus, each summand $a_s \vec{A}_s \vec{H}_s$ integrates both temporal distance and heterogeneous semantic structure. Since $\vec{A}_s$ varies across sensor types and $a_s$ across time, no factorization exists such that $\vec{H}_\text{nei}$ is expressed as $g(a_s) \cdot h(\vec{H}_s)$ without loss of generality. 
\end{proof}

% \begin{corollary}
% 
% The temporal aggregation in THGNN emphasizes not only temporally proximate messages but also semantically coherent sensor relationships modulated by heterogeneous types. This heterogeneity-aware temporal encoding enhances the expressivity of the model in real-world RUL prediction scenarios.
% \end{corollary}
\begin{remark}
The temporal aggregation in THGNN emphasizes not only temporally proximate messages but also semantically coherent sensor relationships modulated by heterogeneous types. This heterogeneity-aware temporal encoding enhances the expressivity of the model in real-world RUL prediction scenarios.
\end{remark}

\section{Experiments}

In this section, we conduct comprehensive evaluation on two benchmark datasets for RUL, and demonstrate the superior performance of our proposed approach \model\ through a comparative analysis with state-of-the-art baselines, a model ablation study, a hyperparameter sensitivity analysis, and a model complexity analysis.

\subsection{Datasets}

%To assess our model's effectiveness, we conducted tests using two renowned datasets for RUL prediction:

We experiment with two datasets, namely, C-MAPSS and N-CMAPSS, which track the life cycle of aircraft engines. 

\stitle{C-MAPSS.} This dataset details the degradation of aircraft engines, monitored through 21 sensors measuring various parameters like temperature, pressure, and fan speed. Following earlier studies \cite{chen2020machine, xu2021kdnet}, we exclude sensors with indices 1, 5, 6, 10, 16, 18, and 19 due to their constant values. C-MAPSS encompasses four subsets, each under different operating conditions and fault modes. The training and testing splits of each subset are summarized in Table~\ref{tab:data cmapss}.

\stitle{N-CMAPSS.} This dataset captures the complete life cycle, including climb, cruise, and descent phases of turbofan engines, offering a more comprehensive suite of data than C-MAPSS that involves only standard cruise conditions. N-CMAPSS enhances the degradation modeling fidelity, reflecting more complex real-world factors. We use the DS02 subset from N-CMAPSS, comprising data from 20 channels across nine units. For our training, we include six units (units 2, 5, 10, 16, 18, and 20), while the testing sets comprised units 11, 14, and 15, as specified by Chao et al.~\cite{arias2021aircraft}. The predictive performance is separately tested on each of the three testing units, and then evaluated on the aggregated testing set (denoted as WT) comprising all three units, as shown in Table~\ref{tab:data n-cmapss}.

For both datasets, we follow previous work \cite{wang2023local,chen2020machine, heimes2008recurrent, jayasinghe2019temporal, behera2021generative, zhang2016multiobjective,li2018deep,mo2022multi} for sample extraction, labeling and normalization. 
Note that, given the large sample count in N-CMAPSS post-processing, we subsample the dataset with five subsets of training and testing splits, and conduct experiments on each subset \cite{wang2023local}. The mean performance over these five subsets are reported. Conversely, since C-MAPSS has fewer samples, we directly utilize the preprocessed datasets for repeated experiments, averaging the performance across five runs.

\begin{table}[t]
\centering
\caption{Summary statistics of C-MAPSS.}
\begin{tabular}{lcccc}
\hline
                     & FD001  & FD002  & FD003  & FD004  \\ \hline
\# Training samples  & 13785  & 30736  & 15536  & 33263  \\
\# Testing samples   & 100    & 259    & 100    & 248    \\
\# Sensors           & 14     & 14     & 14     & 14     \\ \hline
\vspace{-2mm}
\label{tab:data cmapss}
\end{tabular}
\end{table}

\begin{table}[t]
\centering
\caption{Summary statistics of N-CMAPSS.}
\begin{tabular}{lcccc}
\hline
                     & Unit 11  & Unit 14  & Unit 15  & WT  \\ \hline
\# Training samples  & 52608  & 52608  & 52608  & 52608  \\
\# Testing samples   & 6630    & 1563    & 4330    & 12523    \\
\# Sensors           & 20     & 20     & 20     & 20     \\ \hline

\vspace{-2mm}
\label{tab:data n-cmapss}
\end{tabular}
\end{table}

\subsection{Experimental Settings}

The C-MAPSS and N-CMAPSS datasets were preprocessed identically to the referenced work \cite{wang2023local}, using the same train-test splits and normalization methods.
In the training phase, we utilize a batch size of 50, and apply the Adam optimizer with learning rates of 0.0005 for the C-MAPSS dataset and 0.0001 for the N-CMAPSS dataset. 
The number of layers is set to one in both the recurrent unit GRU and message-passing GNN. 
We set the training epochs to 50 for both datasets, incorporating an early stopping mechanism with a patience level of 10. Additionally, we need to deal with two critical hyperparameters, 
namely, the hidden neurons of the MLP used by the global mapping function $f(\cdot)$ and the time window size $L$, through a detailed sensitivity analysis to choose their optimal values.

Our approach THGNN
% \footnote{\textcolor{blue}{Anonymous code link at \url{https://github.com/666666abc/THGNN/tree/main}}} 
has been implemented using PyTorch 1.9, and all experiments have been conducted on a NVIDIA GeForce RTX 3090 GPU. For evaluation, we employ two metrics, including Root Mean Square Error (RMSE) and the Score function \cite{xu2021kdnet}, as follows.

\begin{align}
\text{RMSE} &= \textstyle\sqrt{\frac{1}{K}{\sum_{i=1}^{K} (\hat{y}_i - y_i)^2}},\\
\text{Score}_i &= \begin{cases} 
\exp(-\frac{\hat{y}_i - y_i}{13}) - 1, & \text{if } \hat{y}_i < y_i \\
\exp(\frac{\hat{y}_i-y_i}{10}) - 1, & \text{else}
\end{cases},
\nonumber\\
\text{Score} &= \textstyle \sum_{i=1}^{K} \text{Score}_i,
\end{align}

\subsection{Comparative Analysis}

In this section, we present a quantitative comparison against state-of-the-art baselines, followed by a visual comparison to the ground truth.

\stitle{Quantitative Comparison.}
our method is benchmarked against state-of-the-art models in RUL prediction. Previous studies \cite{sateesh2016deep,chen2020machine, miao2019joint} have highlighted the advantages of deep learning-based models over traditional shallow learning approaches. Hence, our focus is on deep learning models, which include both conventional temporal encoder-based techniques and graph-based approaches.

From conventional deep learning methods, we consider Hybrid \cite{chen2020machine}, BSLTM \cite{huang2019bidirectional}, TCMN \cite{jayasinghe2019temporal}, KDnet \cite{xu2021kdnet}, SRCB \cite{hong2020multivariate}, 2D-CNN \cite{zhao2020double}, and 3D-CNN \cite{wang2021deep}, with SRCB, 2D-CNN, and 3D-CNN integrating sensor correlations into their frameworks. Additionally, transformer-based methods like Transformer \cite{mo2021remaining}, Informer \cite{zhou2021informer}, Autoformer \cite{wu2021autoformer}, and Crossformer \cite{zhang2022crossformer} are also compared. 
The combination of Informer and RevIN \cite{kim2022reversible}, a simple yet effective normalization-and-denormalization method for timeseries, is also considered. 
We also benchmark against graph-based methods, including as DAG \cite{li2019directed}, STGCN \cite{wang2021spatio}, HAGCN \cite{li2021hierarchical}, MAGCN \cite{chen2023multi}, LOGO \cite{wang2023local}, and TKGIN \cite{zhang2023temporal}. 
Finally, we have also devised a baseline model, named THAN, which incorporates temporal message passing with a widely used heterogeneous graph neural network called HAN \cite{wang2019heterogeneous}. Except THAN, the results of the baselines in Tables~\ref{table:performance c-mapss} and \ref{table:performance n-cmapss} are reproduced from a prior study \cite{wang2023local}, since we adopt the exact same data preprocessing steps.

The results on the C-MAPSS and N-CMAPSS datasets are reported in 
Table~\ref{table:performance c-mapss} and \ref{table:performance n-cmapss}, respectively. 
  Notably, the C-MAPSS dataset is characterized by limited scope and data coverage, smaller dataset size, and simpler degradation patterns. Despite these limitations, \model\ achieves competitive performance in terms of RMSE and Score, coming close to the best-performing baseline.

 Meanwhile, on the N-CMAPSS dataset, THGNN consistently achieves the best performance across the three testing units as well as their aggregation WT on both metrics.  
 Unlike C-MAPSS, N-CMAPSS incorporates \textbf{more complexities that better mimic real-world engine conditions}, as noted by Chao et al.~\cite{arias2021aircraft}, thereby presenting a more challenging dataset for analysis. 
 In particular, it captures more comprehensive phases including climb, cruise, and descent phases, whereas C-MAPSS involves only the cruise phase.
 This complexity is further manifested through more intricate temporal and heterogeneous patterns with a greater variety of sensor types.
 Across the three subsets, our method \model\ shows a significant enhancement, improving the predictive performance by 5.7\%--19.2\% in RMSE and 19.8\%--31.6\% in Score, compared to the best baseline method in each case. The \textbf{superior} RUL predictions of our model on the \textbf{more complex} N-CMAPSS engine data imply the \textbf{robustness} of \model\ in potential \textbf{practical} deployment settings.

It is also worth noting that, among the baselines, graph-based methods tend to perform better due to the capture of important spatial dependencies across sensors. 
Transformer-based approaches such as Informer and Crossformer trail closely behind due to their ability to capture long-range dependencies by the Transfomer architecture, while conventional temporal encoding approaches tend to produce less favorable results.
We notice that some baselines, like LOGO, achieve better results than \model\ on certain datasets, often corresponding to simpler scenarios where its coarse-grained approach is sufficient. However, \model\ consistently outperforms LOGO on more complex datasets, such as N-CMAPSS, where its specialized components effectively handle sensor heterogeneity and nuanced temporal dependencies.

\stitle{Visual Comparison.}
% Moreover, we visualize the predicted RUL of our \model\ and the ground-truth RUL on both C-MAPSS and N-CMAPSS datasets, plotted in Figures~\ref{fig:vis our cmapss} and \ref{fig:vis our n-cmapss}, respectively. 
Following prior works \cite{wang2023local, wang2021spatio}, we visualize the predicted RUL of our \model\ and the ground-truth RUL on both C-MAPSS and N-CMAPSS datasets, plotted in Figures~\ref{fig:vis our cmapss} and \ref{fig:vis our n-cmapss}, respectively.
The visual comparison focuses solely on the gap between the proposed method’s predictions and the ground truth, rather than the gap with other baselines. This is because the primary goal of these visualizations is to illustrate how closely the proposed method aligns with the ground truth. The comparative performance of our method against baseline methods is already quantitatively demonstrated in the main experimental results, where metrics like RMSE and Score provide a clear and objective comparison.
Given the large number of samples in the N-CMAPSS testing sets, we selectively visualize 150 random samples from each testing set for clarity. The visual comparisons across all subsets of the two datasets reveal a close alignment between the predicted RUL values and the actual RUL values. This consistency across diverse testing scenarios underscores the robustness and accuracy of our method \model\ in predicting RUL, demonstrating its practical applicability and effectiveness in real-world settings.

\begin{table*}[tbp]
\scriptsize
% \footnotesize
\centering
\caption{Performance comparison to baseline methods on C-MAPSS (bold: best; underline: runner-up).}
\addtolength{\tabcolsep}{-5pt}
\begin{tabular}{lcccccccccccccccccccc}
\toprule
\multirow{2}{*}{\textbf{Models}} & \multicolumn{4}{c}{\textbf{FD001}} & \multicolumn{4}{c}{\textbf{FD002}} & \multicolumn{4}{c}{\textbf{FD003}} & \multicolumn{4}{c}{\textbf{FD004}} & \multicolumn{4}{c}{\textbf{Average}} \\
\cmidrule(lr){2-5} \cmidrule(lr){6-9} \cmidrule(lr){10-13} \cmidrule(lr){14-17} \cmidrule(lr){18-21}
                    & \textbf{RMSE} &Rank & \textbf{Score} &Rank & \textbf{RMSE} & Rank& \textbf{Score} &Rank &  \textbf{RMSE} & Rank & \textbf{Score} & Rank & \textbf{RMSE} &Rank & \textbf{Score} &Rank & \textbf{RMSE} &Rank & \textbf{Score} &Rank \\ \midrule
Hybrid      & 13.37 & 10 & 404 & 19 & 14.64 & 14 & 1093 & 16 & 13.91 & 18 & 467 & 18 & 14.82 & 10 & 1038 & 8 & 14.19 & 13 & 751 & 14 \\
BLSTM       & 13.48 & 11 & 335 & 14 & 14.24 & 9 & 917 & 11 & 12.62 & 8 & 307 & 10 & 14.90 & 11 & 1037 & 7 & 13.81 & 8 & 649 & 9 \\
TCMN        & 13.76 & 16 & 287 & 9 & 14.33 & 10 & 878 & 7 & 13.25 & 15 & 370 & 15 & 15.15 & 13 & 981 & 5 & 14.12 & 12 & 629 & 6 \\
KDnet       & 13.54 & 12 & 378 & 17 & 14.04 & 7 & 912 & 10 & 13.01 & 13 & 347 & 14 & 15.53 & 14 & 1206 & 14 & 14.03 & 11 & 711 & 13 \\
SRCB        & 13.36 & 8 & 301 & 10 & 14.03 & 6 & 911 & 9 & 12.34 & 5 & 322 & 11 & 15.13 & 12 & 1076 & 11 & 13.72 & 7 & 653 & 10 \\
2D-CNN      & 15.04 & 19 & 367 & 16 & 15.85 & 17 & 1290 & 17 & 13.79 & 17 & 333 & 12 & 18.31 & 17 & 2520 & 18 & 15.75 & 18 & 1128 & 18 \\
3D-CNN      & 14.90 & 18 & 374 & 18 & 15.64 & 16 & 976 & 13 & 13.21 & 14 & 297 & 9 & 16.75 & 15 & 2185 & 16 & 15.13 & 16 & 958 & 17 \\
Transformer & 12.90 & 5 & 316 & 13 & 19.18 & 19 & 4101 & 20 & 13.27 & 16 & 340 & 13 & 22.18 & 19 & 4965 & 19 & 16.88 & 19 & 2431 & 19 \\
Informer    & 13.65 & 13 & 285 & 7 & 14.58 & 13 & 1040 & 15 & 12.99 & 11 & 260 & 5 & 14.66 & 7 & 1071 & 10 & 13.97 & 10 & 664 & 11 \\
Informer+RevIN & 12.97 & 6 & 271 & 5 & 13.85 & 4 & 988 & 14 & 12.34 & 4 & \textbf{247} & 1 & \underline{13.93} & 2 & 1017 & 6 & 13.27 & 3 & 631 & 7 \\
Autoformer  & 28.14 & 20 & 1938 & 20 & 22.00 & 20 & 2260 & 19 & 32.20 & 20 & 2499 & 20 & 38.45 & 20 & 12100 & 20 & 30.20 & 20 & 4699 & 20 \\
Crossformer & \textbf{12.11} & 1 & \textbf{216} & 1 & 14.16 & 8 & 837 & 5 & 12.32 & 3 & 260 & 5 & 14.81 & 9 & 956 & 3 & 13.35 & 4 & 567 & 3 \\\midrule
DAG         & 12.56 & 4 & 249 & 4 & 16.58 & 18 & 1490 & 18 & 12.79 & 9 & 589 & 19 & 19.36 & 18 & 2200 & 17 & 15.32 & 17 & 1132 & 16 \\
STGCN       & 14.55 & 17 & 402 & 17 & 14.58 & 12 & 943 & 12 & 13.06 & 12 & 394 & 16 & 14.60 & 5 & 1065 & 9 & 14.20 & 14 & 701 & 12 \\
HAGCN       & 13.36 & 8 & 283 & 8 & 14.73 & 15 & 844 & 6 & \textbf{12.06} & 1 & 271 & 8 & 14.55 & 4 & 971 & 4 & 13.68 & 6 & 592 & 4 \\
MAGCN       & 13.73 & 15 & 311 & 12 & 14.38 & 11 & 897 & 8 & 12.94 & 10 & \underline{258} & 2 & 14.75 & 12 & 1079 & 12 & 13.95 & 9 & 636 & 8 \\
LOGO        & \underline{12.13} & 2 & \underline{226} & 2 & \textbf{13.54} & 1 & 832 & 4 & \underline{12.18} & 2 & 261 & 7 & 14.29 & 3 & \underline{944} & 2 & \textbf{13.04} & 1 & \underline{566} & 2 \\
TKGIN & 12.69 & 5 & 269 & 6 & 13.99 & 5 & \underline{802} & 2 & 12.46 & 6 & \textbf{257} & 3 & \textbf{13.82} & 1 & \textbf{922} & 1 & \underline{13.24} & 2 & \textbf{563} & 1 \\
THAN        & 13.69 & 14 & 317 & 15 & \underline{13.71} & 2 & \textbf{801} & 1 & 14.18 & 19 & 402 & 17 & 17.02 & 16 & 1674 & 15 & 14.65 & 15 & 799 & 15 \\

\midrule
THGNN (Ours) & 13.15$\pm$0.53 & 7 & 285$\pm$40 & 7 & 13.84$\pm$0.65 & 3 & 806$\pm$105 & 3 & 12.61$\pm$0.42 & 7 & \underline{255}$\pm$35 & 2 & 14.65$\pm$0.73 & 6 & 1166$\pm$201 & 13 & 13.56 & 5 & 628 & 5 \\
\bottomrule
\end{tabular}

\label{table:performance c-mapss}
\end{table*}

\begin{table*}[tbp]
\centering
\scriptsize
\captionsetup{justification=centering}
\caption{Performance comparison to baseline methods on N-CMAPSS (bold: best; underline: runner-up). \\
WT is the integration of the three sub-datasets (Unit 11, Unit 14, and Unit 15).}
\addtolength{\tabcolsep}{-3pt}
\begin{tabular}{lcccccccccccccccc}
\toprule
\multirow{2}{*}{\textbf{Models}} & \multicolumn{4}{c}{\textbf{Unit 11}} & \multicolumn{4}{c}{\textbf{Unit 14}} & \multicolumn{4}{c}{\textbf{Unit 15}} & \multicolumn{4}{c}{\textbf{WT (overall)}}  \\
\cmidrule(lr){2-5} \cmidrule(lr){6-9} \cmidrule(lr){10-13} \cmidrule(lr){14-17}
                & \textbf{RMSE} &Rank & \textbf{Score} &Rank & \textbf{RMSE} &Rank & \textbf{Score} &Rank & \textbf{RMSE} &Rank & \textbf{Score} &Rank & \textbf{RMSE} &Rank &\textbf{Score} &Rank   \\ \midrule
Hybrid     & 7.86 &13& 18531&16          & 8.11&14& 3241&16           & 5.79&14& 4456&13         & 8.06&15& 37028&18   \\
BLSTM      & 8.76&16& 22805&18          & 8.51&15& 2833&14           & 5.41&12& 3901&11           & 8.26&16& 32421&16   \\
TCMN       & 6.03&5& 8442&6           & 7.35&10& 2346&11           & 5.32&10& 3515&9               & 6.65&8& 15407&4    \\
KDnet      & 8.48&15& 9463&11           & 7.93&13 & 1930&4           & 5.76&13& 3977&12           & 8.34&17& 16461&6           \\
SRCB       & 6.64&11 & 8613&8           & 7.34&9 & 2118&9           & 4.93&5 & 3207&5           & 6.70&10 & 20361&14       \\
2D-CNN     & 6.38&9 & 8531&7           & 7.20&7 & 2051&8           & 5.20&8 & 3468&7           & 6.99&13& 17100&12           \\
3D-CNN     & 6.21&7 & 8963&9           & 7.12&6 & 1976&6           & 5.25&9 & 3363&6           & 6.76&11 & 16887&8        \\
Transformer& 5.86&3 & 7725&4           & 7.69&12 & 2397&12           & 5.34&11 & 3391&6           & 6.54&7 & 17075&11    \\
Informer   & 6.28&8 & 8019&5           & 7.67&11 & 2437&13           & 5.03&6 & 3195&4           & 6.24&5& 16066&5    \\
Informer+RevIN & 5.97 & 4 & 7618 & 3     & 7.29 & 8 & 2315 & 10 & 4.78 & 3 & 3035 & 3     & \underline{5.93} & 2 & 15263 & 3 \\
Autoformer & 10.73&19 & 23987&19          & 13.72&20 & 9533&19           & 11.29&20 & 15383&18          & 11.48&20 & 50322&20          \\
Crossformer& 6.89&12 & 11816&13          & 8.58&16 & 2855&15           & 7.39&17  & 5733&14           & 6.87&12 & 16704&7        \\\midrule
DAG        & 9.00&18 & 21282&17          & 9.37&17  & 4211&17           & 7.30&16 & 7970&15           & 8.82&18 & 36780&17        \\
STGCN      & 8.96&17 & 32080&20          & 11.96&19 & 17665&20          & 8.39&19  & 20169&19          & 9.31&19 & 37710&19         \\
HAGCN      & 6.39&10 & 9956&12           & 7.03&5 & 2009&7           & 7.48&18 & 9376&17           & 6.67&9 & 17918&13         \\
MAGCN      & 8.33&14 & 13835&15          & 10.52&18 & 4566&18           & 5.11&7  & 3623&10           & 7.37&14 & 20821&15           \\
LOGO       &\underline{5.73}&2&\underline{7509}&2   &6.72&4&1940&5   &\underline{4.54}&2&\underline{3017}&2   &6.07&3&\underline{15127}&2   \\
TKGIN & 6.07 & 6 & 9458 & 10 & 6.68 & 3 & 1909 & 3 & 7.11 & 15 & 8907 & 16 & 6.34 & 6 & 17022 & 9 \\
THAN        &11.60&20&13272&14       &\underline{6.29}&2&\underline{1691}&2        &4.84&4&3477&8        &6.21&4&17059&10     \\

\midrule
THGNN (Ours)     &\textbf{4.63}$\pm$0.37&1 &\textbf{5138}$\pm$750&1   &\textbf{5.64}$\pm$0.49&1 &\textbf{1356}$\pm$157&1    &\textbf{4.06}$\pm$0.26&1 &\textbf{2402}$\pm$351&1    &\textbf{5.59}$\pm$0.38&1 &\textbf{11183}$\pm$1570&1        \\
\bottomrule
\vspace{-2mm}
\end{tabular}
\label{table:performance n-cmapss}
\end{table*}

\begin{figure*}
   \subfigure[FD001]{
   \centering
   \includegraphics[width=0.28\linewidth]{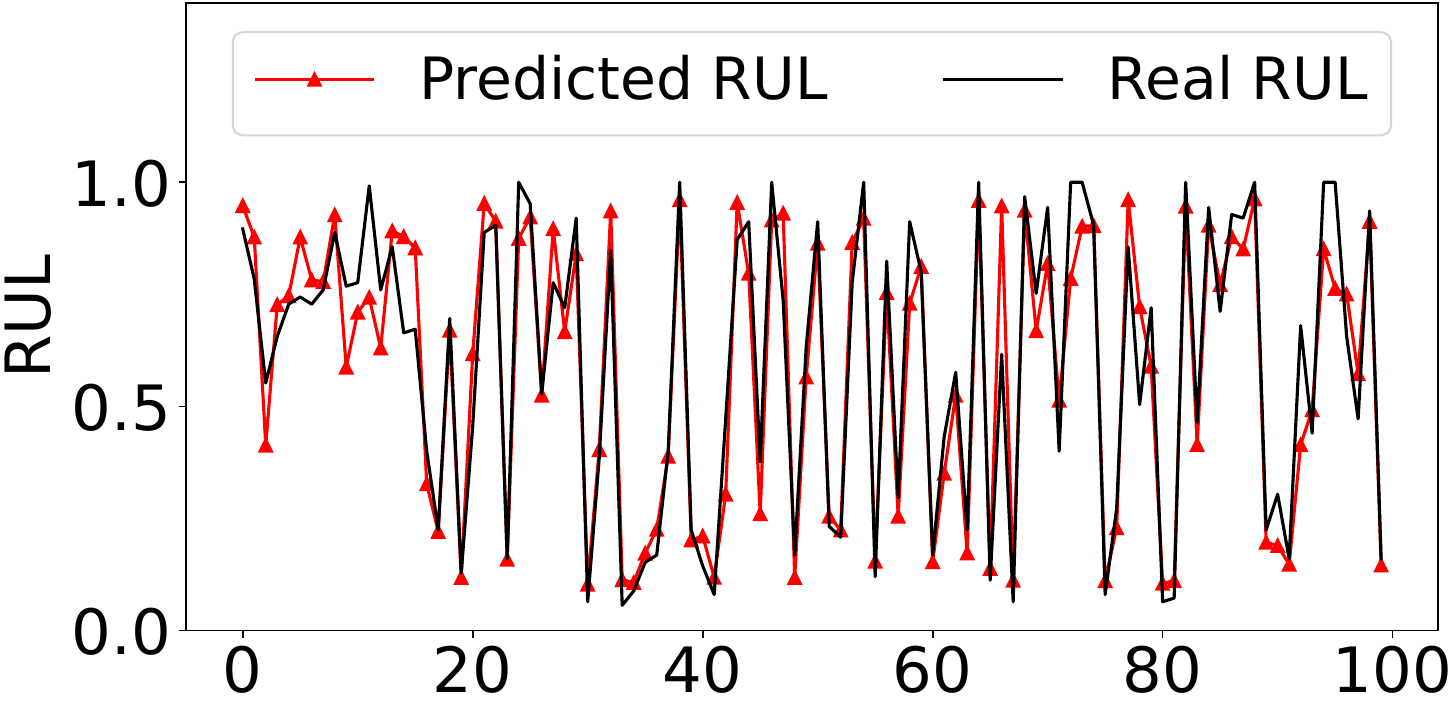}
   % \vspace{-2mm}
   }%
   \subfigure[FD002]{
   \centering
   \includegraphics[width=0.71\linewidth]{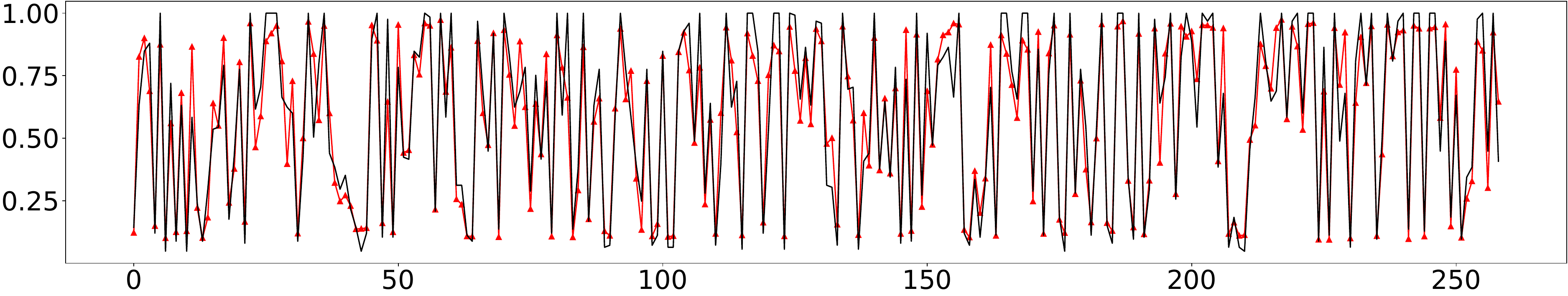}}
   % \vspace{-2mm}
   \subfigure[FD003]{
   \centering
   % \hspace{0.5mm}%
   \includegraphics[width=0.28\linewidth]{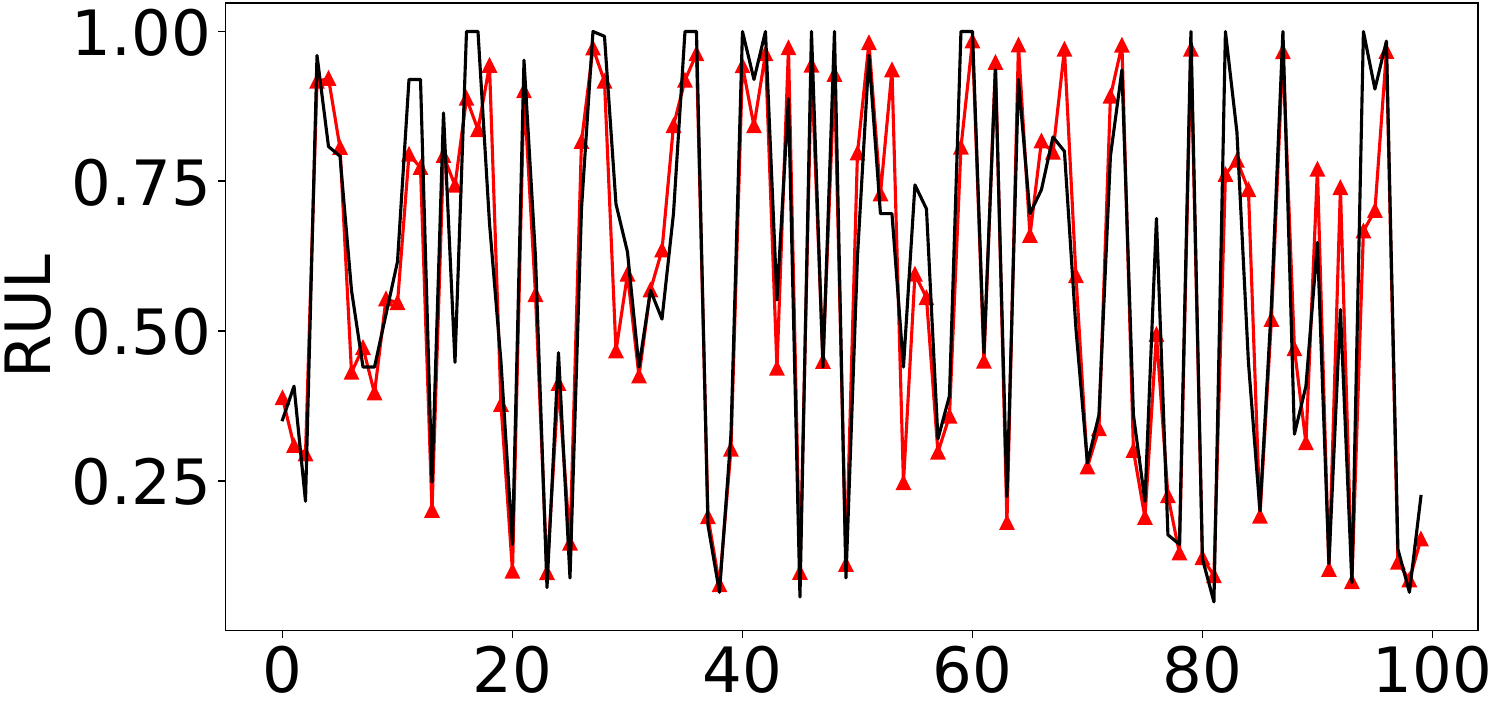}}%
   \subfigure[FD004]{
   \centering
   \hspace{1mm}%
   \includegraphics[width=0.71\linewidth]{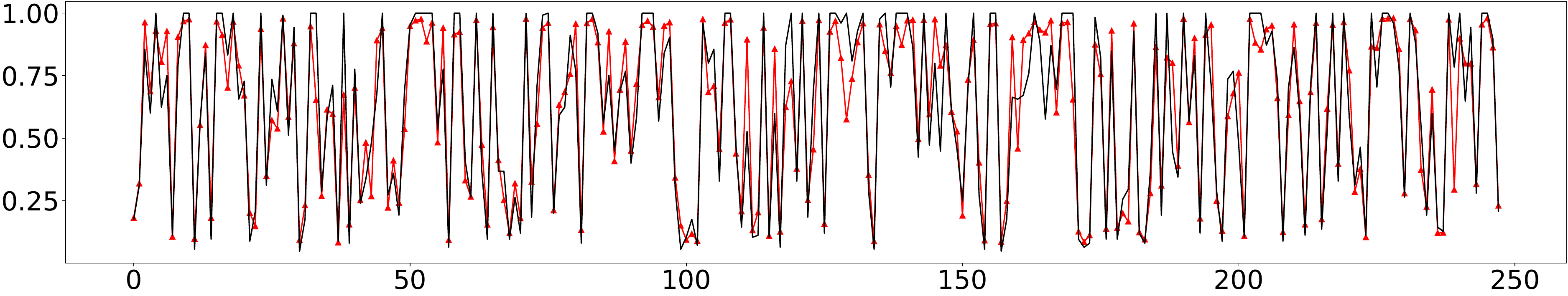}
   }
	\caption{Visual comparison of predicted and actual RUL values across the four subsets of C-MAPSS.}
	\vspace{-2mm}
	\label{fig:vis our cmapss}
\end{figure*}

\begin{figure*}
   \subfigure[Unit 11]{
   \centering
   \includegraphics[width=0.48\linewidth]{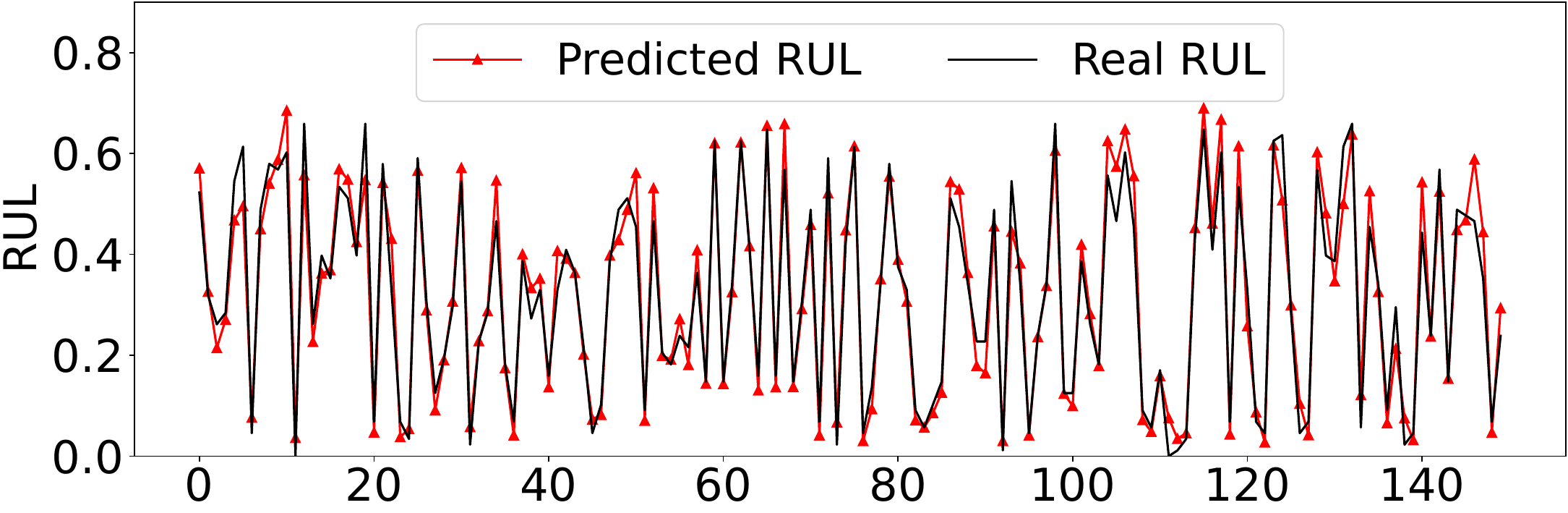}
   % \vspace{-2mm}
   }%
   \subfigure[Unit 14]{
   \centering
   \includegraphics[width=0.48\linewidth]{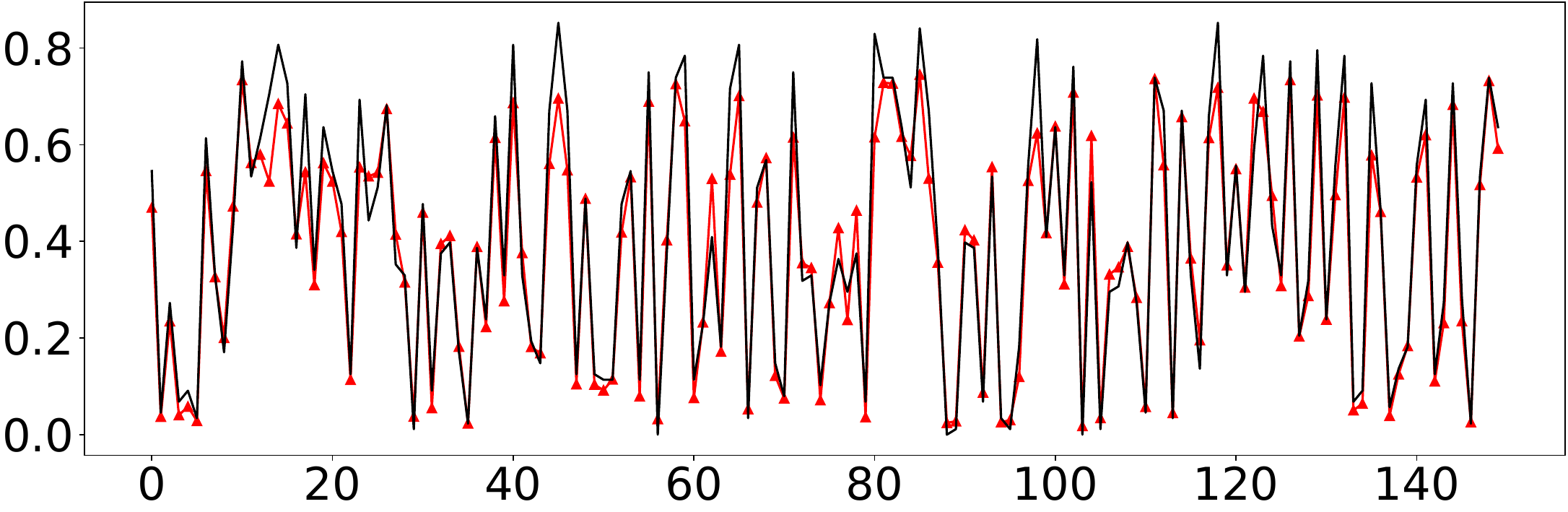}}
   % \vspace{-2mm}
   \subfigure[Unit 15]{
   \centering
   % \hspace{0.5mm}%
   \includegraphics[width=0.48\linewidth]{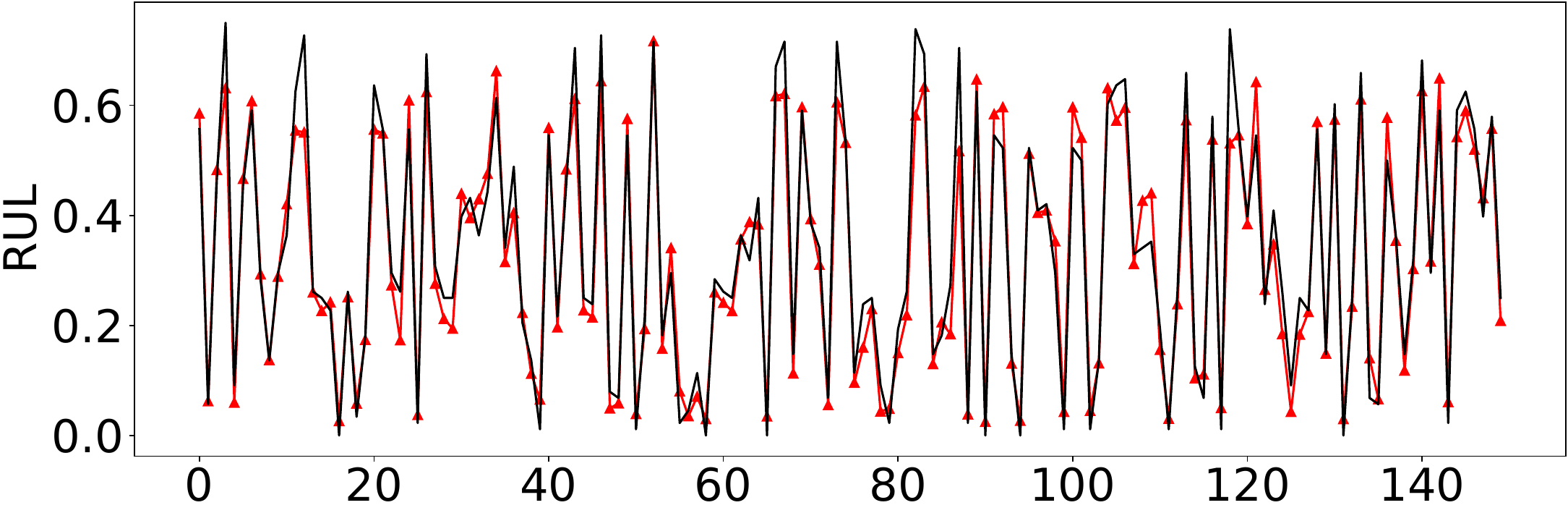}}%
   \subfigure[WT]{
   \centering
   \hspace{1mm}%
   \includegraphics[width=0.48\linewidth]{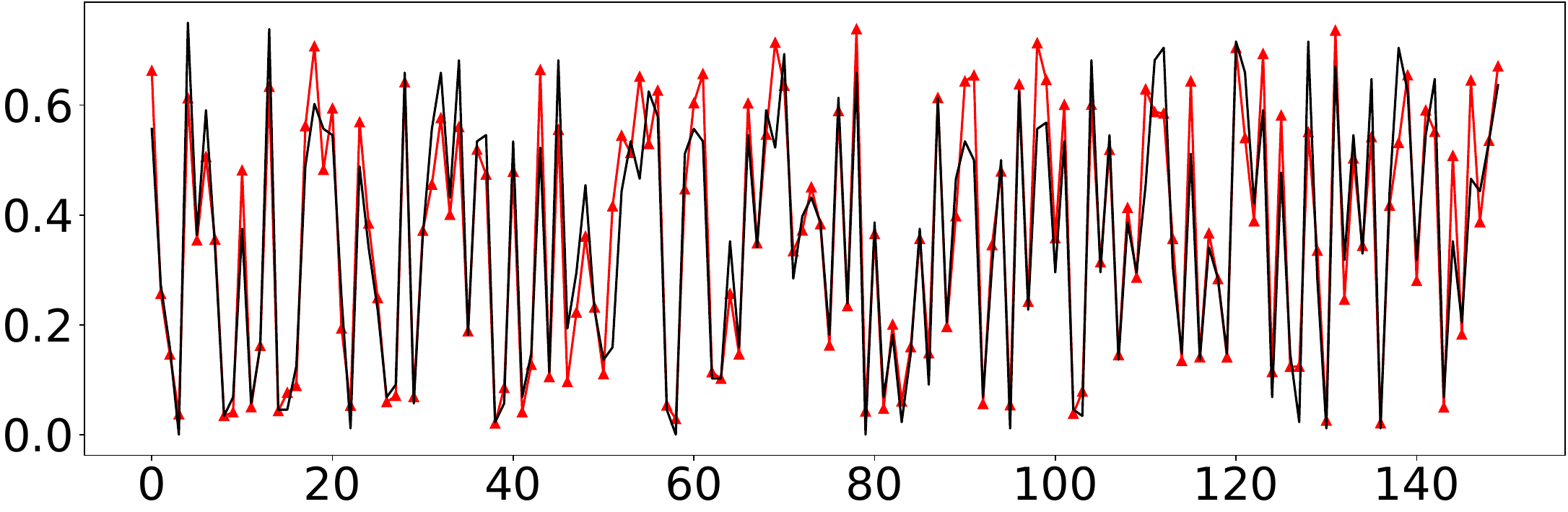}
   }
	\caption{Visual comparison of predicted and actual RUL values across the four subsets of N-CMAPSS.}
	% \vspace{-2mm}
	\label{fig:vis our n-cmapss}
\end{figure*}

\begin{table*}[htb]
\centering
\addtolength{\tabcolsep}{0pt}
\caption{Ablation analysis on C-MAPSS.}
\begin{tabular}{lcccccccccc}
\toprule
\multirow{2}{*}{\textbf{Variants}} & \multicolumn{2}{c}{FD001} & \multicolumn{2}{c}{FD002} & \multicolumn{2}{c}{FD003} & \multicolumn{2}{c}{FD004} & \multicolumn{2}{c}{Average}\\ 
\cmidrule(lr){2-3} \cmidrule(lr){4-5} \cmidrule(lr){6-7} \cmidrule(lr){8-9}\cmidrule(lr){10-11}
& RMSE & Score & RMSE & Score & RMSE & Score & RMSE & Score & RMSE & Score\\
\midrule

No Temporal Aggregation \& Heterogeneity &13.69&301&14.78&925&12.80&253&17.37&1570 &14.66&762\\
No Heterogeneity &12.93&272&14.46&832&15.93&991&17.05&1551 &15.09&912\\

No FiLM &13.69&317&13.71&801&14.18&402&17.02&1674 &14.65&799\\
No MLP as Mapping Function &13.41&304&13.81&850&12.80&300&14.69&1237 &13.68&673\\
No Hawkes &13.78&324&13.98&887 &12.98&376&15.39&1283 &14.03&718\\
\midrule
\model &13.15&285&13.84&806&12.61&255&14.65&1166 &13.56&628\\
\midrule
Snapshot (patch size = 50) &14.53&374&14.87&931 &13.12&483&16.68&1436 &14.80&806\\
Snapshot (patch size = 25)&13.87&383&14.59&917&13.15&464&15.79&1338&14.35&776\\
Snapshot (patch size = 5) &13.05&287&14.08&823&12.97&375&14.68&1206&13.71&673\\
GRU Replaced by Attention &13.99&314&14.01 &819&13.34&289&14.77&1228
&14.03&663\\
% Self-attention for temporal modeling &13.99&314&14.01 &819&13.34&289&14.77&1228
% &14.03&663\\

\bottomrule
\end{tabular}
\label{table:abla c-mapss}
\end{table*}

\begin{table*}[htb]
\centering
\addtolength{\tabcolsep}{4pt}
\caption{Ablation analysis on N-CMAPSS.}
\begin{tabular}{lcccccccc}
\toprule
\multirow{2}{*}{\textbf{Variants}} & \multicolumn{2}{c}{Unit 11} & \multicolumn{2}{c}{Unit 14} & \multicolumn{2}{c}{Unit 15} & \multicolumn{2}{c}{WT (overall)} \\ 
\cmidrule(lr){2-3} \cmidrule(lr){4-5} \cmidrule(lr){6-7} \cmidrule(lr){8-9}
& RMSE & Score & RMSE & Score & RMSE & Score & RMSE & Score \\
\midrule

No Temporal Aggregation \& Heterogeneity &6.97&9457&7.06&2159&5.25&3654&7.04&16391 \\
No Heterogeneity &6.47&10266&5.73&1520&4.21&2601&6.20&15134 \\
No FiLM &11.60&13272&6.29&1691&4.84&3477&6.21&17059 \\
No MLP as Mapping Function &4.72&5215&5.69&1393&3.91&2382&6.90&20683 \\
No Hawkes &6.34&8461&6.21&1638 &5.04&2897&6.09&14932 \\
\midrule
\model &4.63&5138&5.64&1356&4.06&2402&5.59&11183 \\
\midrule
Snapshot (patch size = 50) &6.78&9783&6.59&1784 &5.27&3053&6.25&16275 \\
Snapshot (patch size = 25) &6.39&8863&6.34&1678&5.22&2901&6.12&15389\\
Snapshot (patch size = 5) &6.13&7758&5.56&1399&4.63&2670&5.89&13940\\
GRU Replaced by Attention &5.45&5837&5.78&1492 &4.58&2739&5.83&13689 \\

\bottomrule
\end{tabular}
\label{table:abla n-cmapss}
\end{table*}

\subsection{Ablation Study}

In the ablation study, we aim to assess the effectiveness of the proposed components within our model \model. These components include (1) aggregating historical information prior to the target time $T$; (2) incorporating sensor heterogeneity in graph construction;
(3) universal sensor and time adaptation by FiLM;
(4) employing a two-layer bottleneck-shaped MLP for the global mapping function;
(5) Hawkes Process;
(6) GRU.

The variants are designed to incrementally incorporate more components, showcasing the impact of each on the overall performance, as follows.

\begin{itemize}
    \item \emph{No Temporal Aggregation \& Heterogeneity}: Utilizing only the temporal representation matrix at the target time $T$ for constructing the graph and subsequent message passing. It does not aggregate historical information prior to time $T$ and disregards the heterogeneity of different sensor types when constructing edges. Essentially, \model-T-H functions as a basic message-passing GNN, lacking temporal and heterogeneous modeling.

    \item \emph{No Heterogeneity}: Building upon \emph{No Temporal Aggregation \& Heterogeneity}, it incorporates historical information by aggregating the $L-1$ time steps before the target time $T$ as well. This addition enables the model to leverage temporal patterns, providing a richer context for graph construction and message passing.

    \item \emph{No FiLM}: Further improving over \emph{No Heterogeneity}, it considers the heterogeneity of node types during edge construction. Unlike \emph{No Heterogeneity}, \emph{No FiLM} introduces individual mapping functions for each node type as indicated in Eq.~\eqref{eq:hete functions}. That means it does not employ the \emph{FiLM} layer to achieve universal adaptation across nodes and time steps, which limits its ability to modulate node representations dynamically based on fine-grained node and temporal contexts.

    \item \emph{No MLP as mapping function}: It employs the FiLM layer to tackle the issue of node heterogeneity, further enhancing \emph{No FiLM}. However, it implements the global mapping function based on a lightweight \textbf{fully connected layer}. Although it marginally reduces the parameters and model complexity, its learning capacity is also limited.

    \item \emph{No Hawkes}: Everything else is the same as \model, but instead of using Hawkes Process, this variant \textbf{considers all previous time steps equally, regardless of their temporal distance}.
    
    \item \emph{\model}: This is our full model, which substitutes the fully connected layer in \emph{No MLP as mapping function} with a two-layer MLP as the global mapping function. 

    \item \emph{Snapshot}: Everything else is the same as \model, but instead of aggregating the historical neighbourhood information of $L$ time steps, this variant \textbf{compresses the information of $L$ time steps into a snapshot} by mean pooling. Only \textbf{one adjacency matrix} is involved in predicting the RUL at time $T$. Here the patch size equals $L$, the time steps to be compressed.
    
    \item \emph{GRU Replaced by Attention}: Everything else is the same as \model, but we do not use GRU as the temporal layer, as shown in Figure 2 (a). Instead, we use self-attention mechanism to be the temporal layer. 

\end{itemize}

We report the results of the ablation analysis on the C-MAPSS and N-CMAPSS datasets in  Tables~\ref{table:abla c-mapss} and \ref{table:abla n-cmapss}, respectively. 
In general, we observe a gradual improvement in the overall performance of the variants on both datasets as we progress from the \emph{No Temporal Aggregation \& Heterogeneity} variant to the full \model. One notable exception occurs with \emph{No FiLM}, which often produces results inferior to \emph{No Heterogeneity}. A potential reason is that \emph{No FiLM} introduces one distinct mapping function for each type of sensor, significantly increasing the number of learnable parameters and becoming prone to overfitting. 
This demonstrates that handling node heterogeneity in a na\"ive way may hurt the performance.
Overall, across the entire C-MAPSS dataset, the average performance of \model\ is superior to all ablated variants. And on the more comprehensive WT subset, which integrates all units, \model\ outperforms all ablated variants, demonstrating its robustness and ability to handle diverse temporal and heterogeneous patterns effectively. 
The ablation analysis validates the contribution of each proposed components.

\subsection{Hyperparameter Sensitivity Analysis}

\stitle{Hidden Neurons in Global Mapping.}
In our model, we employ a two-layer bottleneck-shaped MLP for the global mapping function $f(\cdot)$. As the global mapping function is a key design of our approach, the choice of the hidden dimension in this MLP is crucial. To investigate  the impact of the hidden neurons, we conduct a series of experiments in Figure~\ref{fig:hid_dim}. We observe that simply increasing the hidden dimension is not the most effective strategy on the smaller C-MAPSS dataset, as it may lead to overfitting. Instead, a smaller dimension size, such as 4, proves to be sufficient. Conversely, for the larger and more complex N-CMAPSS dataset, it is beneficial to increase the hidden dimension to a moderate size like 32 to enhance the model capacity.

\begin{figure}
   \subfigure[RMSE, C-MAPSS]{
   \centering
   \includegraphics[width=0.48\linewidth]{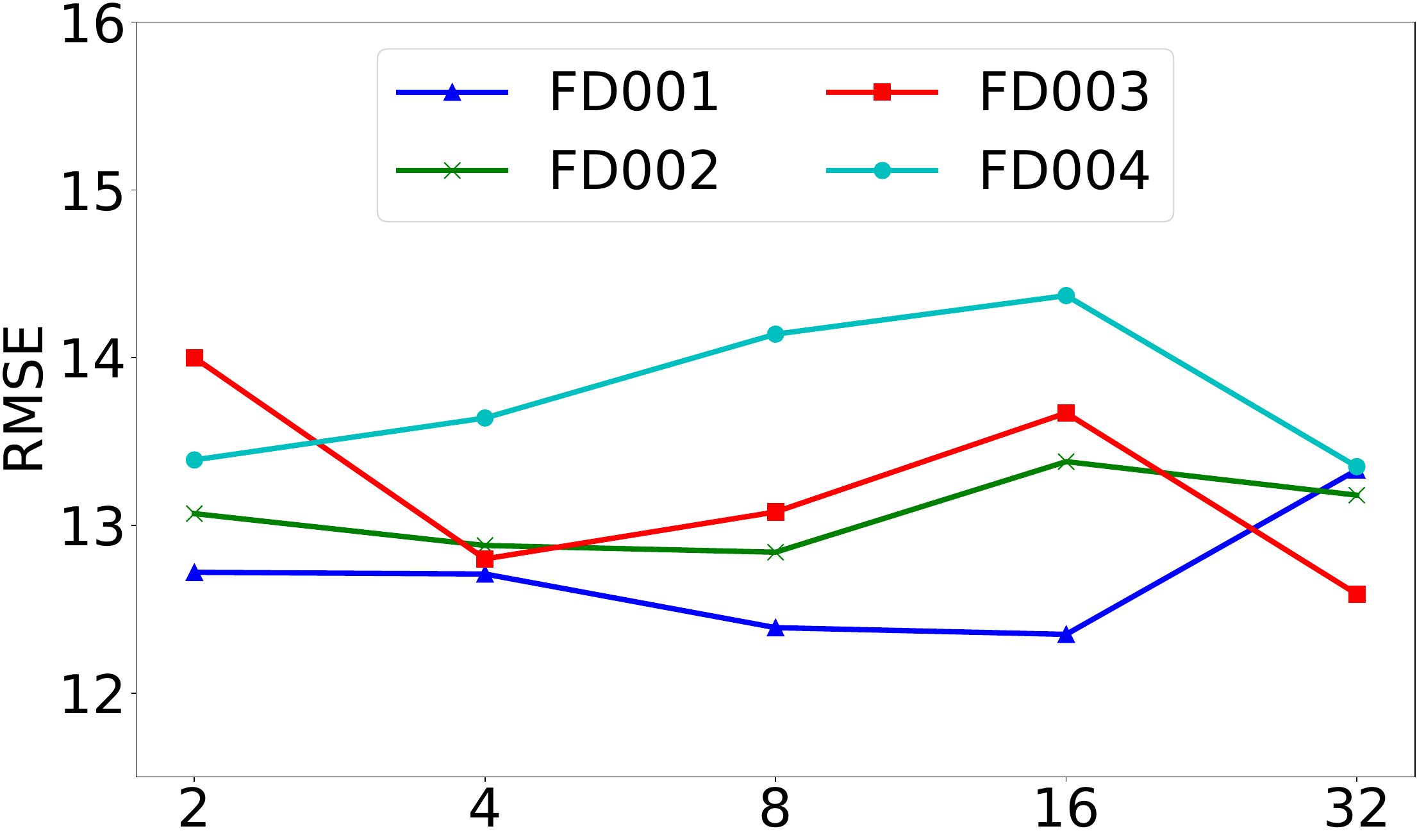}
   }%
   \subfigure[Score, C-MAPSS]{
   \centering
   \includegraphics[width=0.48\linewidth]{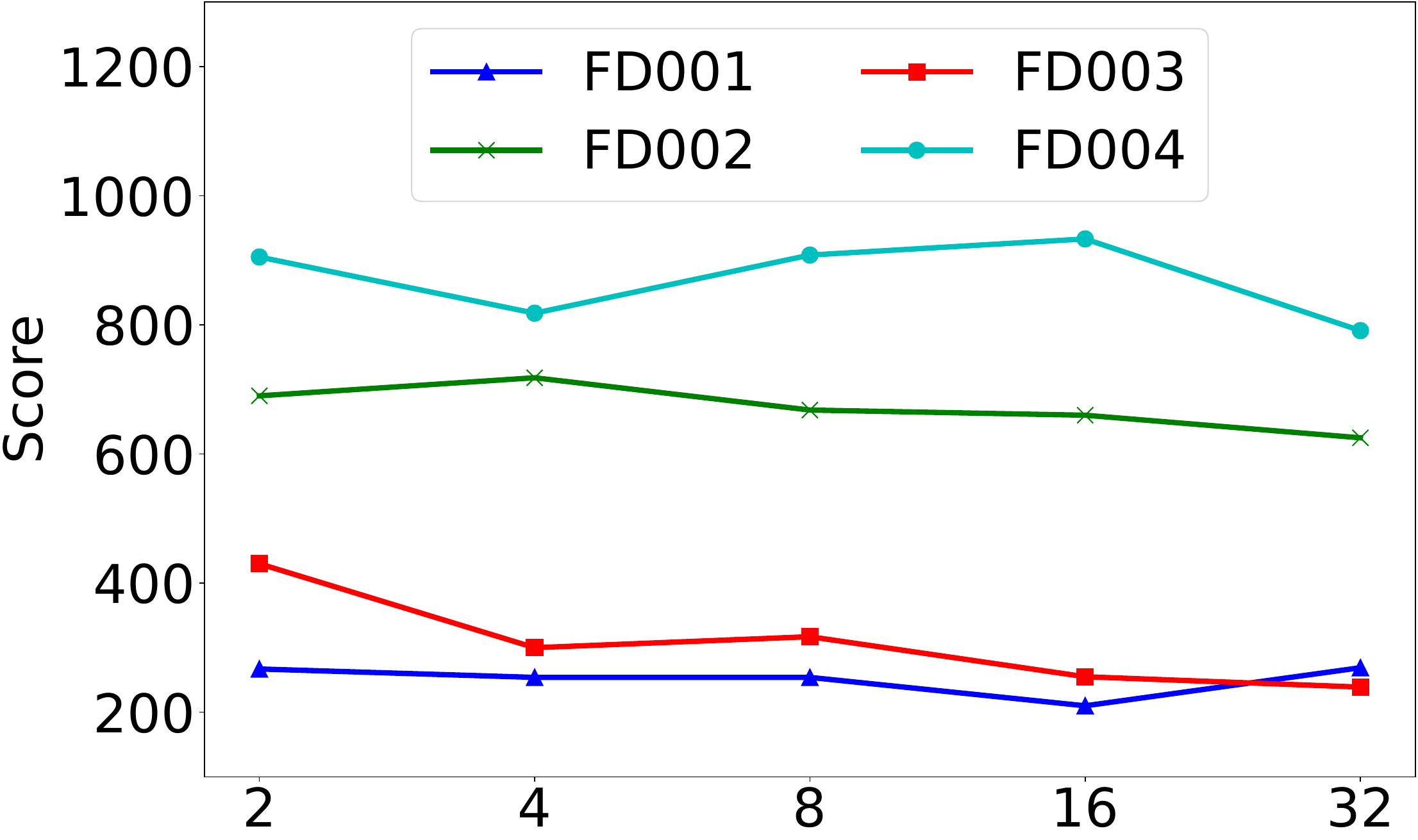}}
   % \vspace{-2mm}
   \subfigure[RMSE, N-CMAPSS]{
   \centering
   % \hspace{0.5mm}%
   \includegraphics[width=0.48\linewidth]{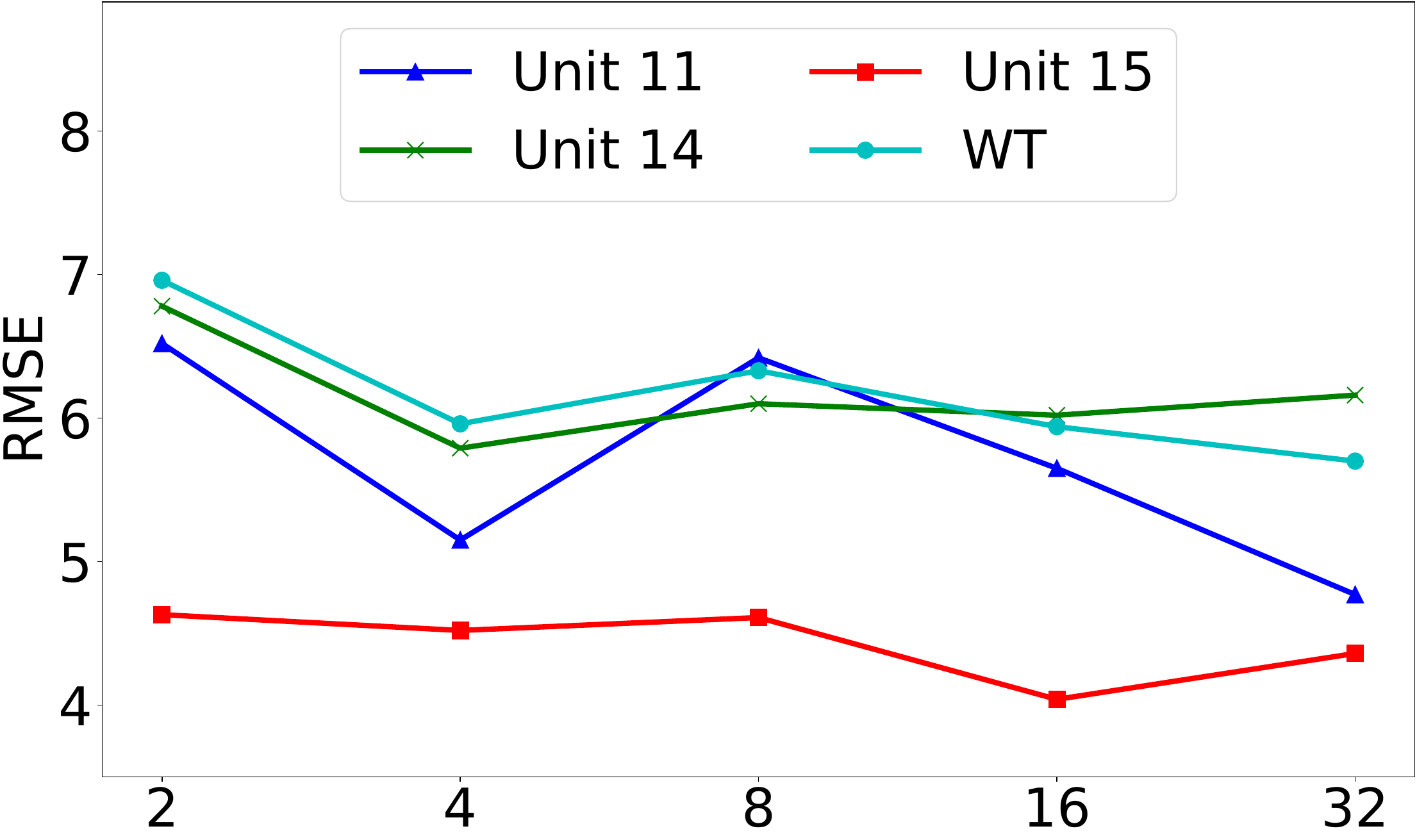}}%
   \subfigure[Score, N-CMAPSS]{
   \centering
   \hspace{1mm}%
   \includegraphics[width=0.48\linewidth]{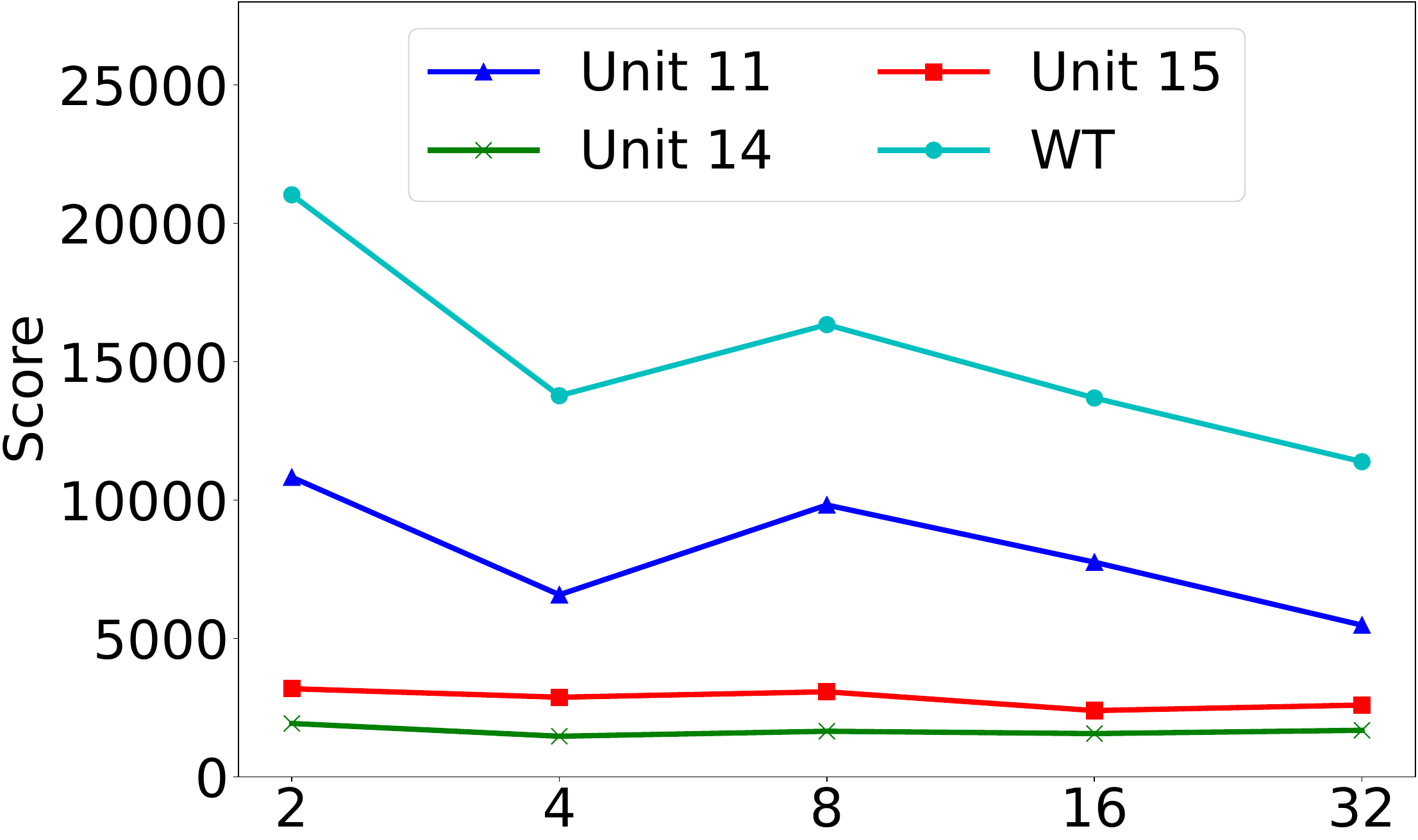}
   }
   % \vspace{-3mm}
	\caption{Hidden dimension of global mapping MLP.}
	% \vspace{-2mm}
	\label{fig:hid_dim}
\end{figure}

\stitle{Time Window Size.}
Our model aggregates information from historical neighbors, and thus the length of history, i.e., the size of the time window $L$, becomes an important consideration. To evaluate the effect of $L$ on the predictive performance, we conduct experiments that vary the value of $L$ between 10 and 90. As depicted in Figure~\ref{fig:time window}, we  observe that increasing the size of the time window $L$ naturally enlarges the pool of historical neighbors, leading to more accurate  predictions as anticipated on C-MAPSS (similar trends can also be observed on N-CMAPSS). However, a larger time window size requires more memory and computational time. Thus, a balance must be struck between model performance and computational efficiency. It is noted that beyond a certain point around $L=50$, the improvement in model performance becomes marginal. Therefore, $L=50$ represents a balanced choice that considers both effectiveness and efficiency.

\begin{figure}
   \subfigure[Impact on RMSE]{
   \centering
   \includegraphics[width=0.48\linewidth]{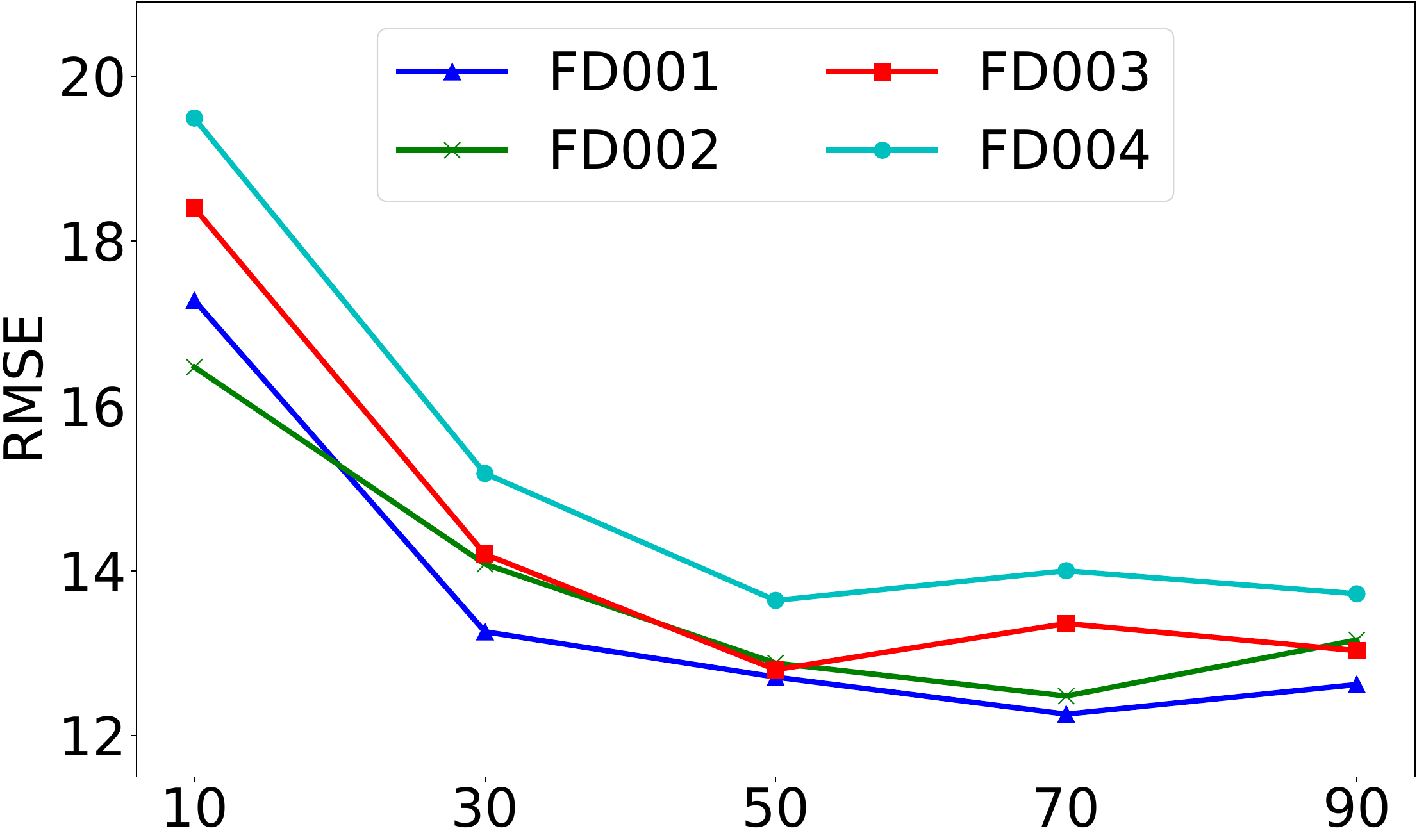}
   % \vspace{-2mm}
   }%
   \subfigure[Impact on Score]{
   \centering
   \includegraphics[width=0.48\linewidth]{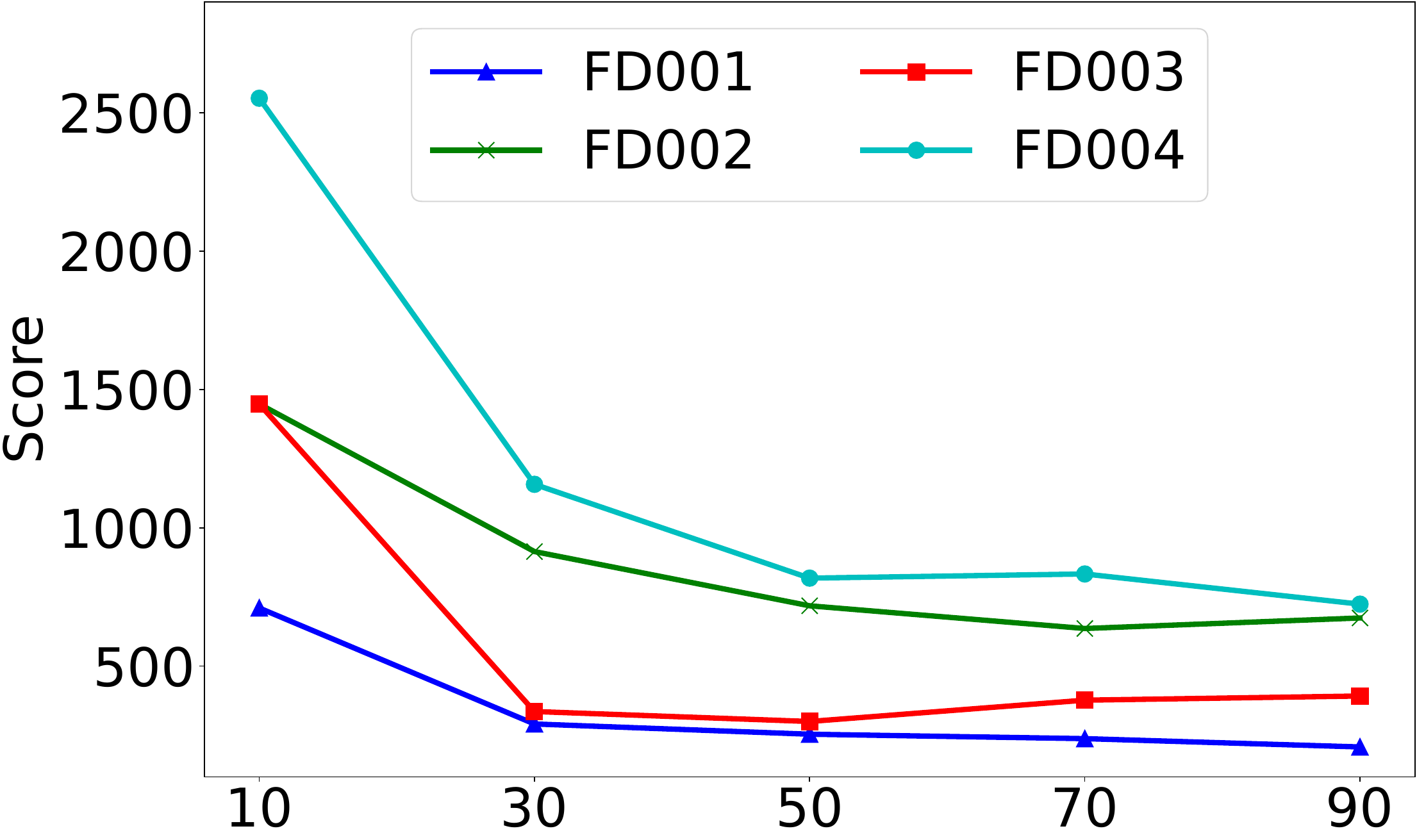}}
   % \vspace{-2mm}
   % \vspace{-3mm}
	\caption{Impact of Time window size $L$ on C-MAPSS.}
	% \vspace{-2mm}
	\label{fig:time window}
\end{figure}

\subsection{Model Complexity Analysis}
% \wen{need rewrite}
The complexity of a model is a key factor in real-world deployments. Furthermore, when comparing different models, it is also important to understand if the increased performance is simply a result of increased model complexity instead of a better model design. 
The overall time complexity of \model\ is derived by analyzing its key operations, including graph construction, message passing, temporal modeling (via GRU), FiLM layer transformations, and temporal aggregation. The dominant terms result in: $\mathcal{O}(T \cdot N^2 \cdot d + T \cdot N \cdot d^2)$. Here, $T$ represents the number of time steps, $N$ the number of sensors, and $d$ the embedding dimension. Similarly, for the overall space complexity,  the storage requirements for adjacency matrices, node embeddings, and hidden states result in an overall space complexity of: $\mathcal{O}(T \cdot N^2 + T \cdot N \cdot d)$. 
What's more, we examine the complexity of our approach \model\ in comparison to several competitive baselines on the C-MAPSS dataset. 
We utilize standard metrics including Floating-point Operations Per Second (FLOPs), the total number of learnable parameters, and  inference time per testing sample to gauge model complexities.

The detailed comparison is presented in Table~\ref{tab:complexity}. 
In terms of FLOPs, \model\ is slightly more expensive than the most lightweight models (e.g., BLSTM and LOGO). This increased demand arises from our integration of RNN and GNN architectures. Despite this, the difference in FLOPs can be deemed marginal given the superior performance of \model, especially on the more realistic N-CMAPSS dataset as shown in Table~\ref{table:performance n-cmapss}.
Despite of its superior performance, our \model\ entails the least number of learnable parameters, and thus carries a low risk of overfitting. The parameter efficiency is due to our design that includes the bottleneck MLP structure and the FiLM layer. 
And \model\ achieves the second-fastest inference speed among all methods, highlighting its efficiency despite incorporating fine-grained temporal modeling and heterogeneous graph structures.

\begin{table}[ht]
\centering
\addtolength{\tabcolsep}{2pt}
\caption{Comparison of model complexity between \model\ and the baselines (bold: best; underline: runner-up).}
\begin{tabular}{lrrr}
\toprule
\textbf{Models} & \textbf{FLOPs} & \textbf{\# Parameters} & \textbf{Inference/ms}  \\
\midrule
Hybrid     & 1,087,500   & 32,207  &129\\
BLSTM      & \textbf{880,338}     & 36,613  &  101\\
TCMN       & 3,982,694   & 193,271  & 126\\
KDnet      & 2,423,360   & 114,753 &  \textbf{68}\\
SRCB       & 5,177,824   & 492,335  & 139\\
2D-CNN     & 1,863,296   & 56,001  &  113\\
3D-CNN     & 154,160,562 & 106,309 &  458\\
Transformer& 2,556,400   & 33,207  &  142\\
Informer   & 4,861,920   & 187,073  & 130\\
Informer+RevIN &4,894,020 & 187,715 & 133\\
Autoformer & 1,442,586   & 39,524  &  141\\
Crossformer& 34,746,894  & 1,194,907 & 128\\
\midrule
DAG        & 1,131,802   & 69,448  &  135\\
STGCN      & 1,121,822   & 47,981  &  171\\
HAGCN      & 9,759,976   & 30,340  & 213\\
MAGCN      & 2,213,406   & 32,464  &  131\\ 
LOGO       & \underline{1,078,876}   & 32,009 & 123\\
TKGIN &11,836,416&446,464&374 \\
THAN &3,724,161& \underline{20,644} & 130\\

\midrule
\model     & 1,496,513   & \textbf{10,458}  &  \underline{79}\\
\bottomrule
\end{tabular}
\label{tab:complexity}
\end{table}

\subsection{Analysis of Capacity to Respond to Heterogeneity}

\begin{figure}[t]
\centering
\includegraphics[width=0.5\linewidth]{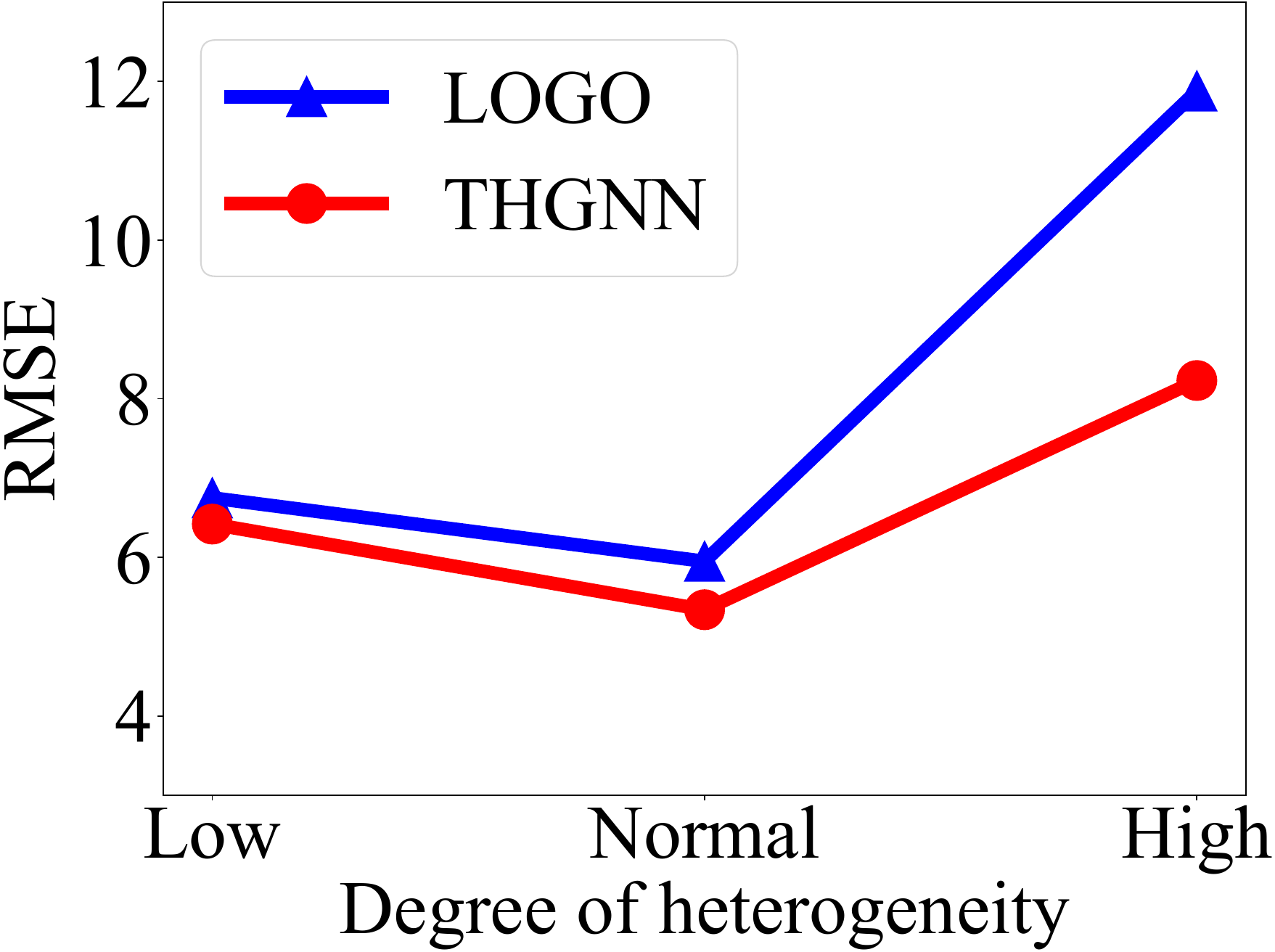}
\caption{Analysis of capacity to respond to heterogeneity}
\label{fig:hetero study}
\end{figure}

To assess the robustness of THGNN under varying levels of heterogeneity, we selected three groups of test data from the N-CMAPSS test set, characterized by low, medium, and high levels of heterogeneity.
We adjusted the connections between sensors to simulate different levels of heterogeneity: eliminating edges between sensors with different node types results in low heterogeneity, keeping the original edge structure represents normal heterogeneity, and removing edges between sensors with the same node type creates high heterogeneity.
As shown in Figure~\ref{fig:hetero study}, THGNN demonstrates strong performance across all scenarios, compared with LOGO, one competitive baseline.
THGNN performs particularly well in low-heterogeneity conditions due to its ability to exploit stronger correlations between neighboring sensors.
Even in high-heterogeneity conditions, the flexibility of the FiLM-based adaptation layer allows THGNN to maintain competitive performance, with only a gradual increase in RMSE.
LOGO, by contrast, lacks components to account for heterogeneity, resulting in a significant degradation in performance as heterogeneity increases.
We also observed an interesting phenomenon: under the low-heterogeneity condition, both models perform slightly worse compared to the normal-heterogeneity condition. This is likely because eliminating edges between sensors of different node types reduces the overall number of edges, which in turn limits the flow of information.

\section{Conclusion}
% \wen{need rewrite more}
In this work, we introduced a novel Temporal and Heterogeneous Graph Neural Network (\model) for the prediction of RUL in system health monitoring. 
\model\ has the following strengths:
(1) The fine-grained temporal modeling captures subtle temporal dynamics, leading to significant performance improvements;
(2) The heterogeneous graph construction, combined with the FiLM-based adaptation layer, allows our model to handle diverse sensor types effectively;
(3) Despite its advanced modeling capabilities, the method requires less parameters and achieves competitive inference speed, making it practical for real-time applications.
We conducted extensive experiments on both the C-MAPSS and N-CMAPSS datasets to evaluate the effectiveness of our method. 
Most notably, we achieved significant improvements over the best prior method on the N-CMAPSS dataset, ranging 5.7\%--19.2\% in RMSE and 19.8\%--31.6\% in Score.
% The proposed model aggregates historical neighbors to capture the temporal dynamics and spatial correlations between sensors in a fine-grained manner. In particular, our aggregation of historical neighbors is inspired by the Hawkes process,  and we exploit the heterogeneity of diverse sensors using a FiLM layer. 
% Finally, we conducted extensive experiments on both the C-MAPSS and N-CMAPSS datasets to evaluate the effectiveness of our method. 
% Most notably, we achieved significant improvements over the best prior method on the N-CMAPSS dataset, ranging 5.7\%--19.2\% in RMSE and 19.8\%--31.6\% in Score.
For future directions, the fully connected graphs used in this study could be optimized to improve scalability for larger-scale applications. One promising avenue is the use of threshold-based edge sparsification, which could enhance both efficiency and alignment with conventional GNN methodologies.
In the specific context of machine RUL prediction, the sensor set is typically fixed, with new sensors rarely introduced under normal operating conditions. Thus, inductive scenarios are generally less relevant for this task. Nevertheless, in cases such as sensor replacements or system upgrades, the inductive capability of THGNN presents a valuable opportunity for further exploration.

% As for the future directions, first, the graphs we construct are fully connected graphs, so optimizing graph topology could be valuable for larger-scale applications. Possible future direction can be using threshold-based edge sparsification, to improve scalability and alignment with traditional GNN practices.
% Second, in the specific context of machine RUL prediction, the set of sensors in a machine is typically fixed, and new sensors are rarely introduced under normal operating conditions. As a result, inductive scenarios are less relevant for this task. 
% However, in certain cases, such as sensor replacement or system upgrades, the inductive capability of THGNN could prove valuable, worth of exploring.
% \wen{noted edge sparsification as a potential enhancement for future iterations of our model to improve scalability and alignment with traditional GNN practices.}
% We conducted extensive experiments on both C-MAPSS and N-CMAPSS datasets to validate the effectiveness of our method. Notably, we achieved significant improvements on the N-CMAPSS datasets, which are more complex and closely resembles real-world conditions. Specifically, our approach resulted in a 10.6\% enhancement in Root Mean Square Error (RMSE) and a 20.4\% increase in Score.

\section*{Acknowledgments}
This research is supported by the Agency for Science, Technology and Research (A*STAR) under its AME Programmatic Funds (Grant No. A20H6b0151), 
% This research / project is supported by the Ministry of Education, Singapore, under its Academic Research Fund Tier 2 (Proposal ID: T2EP20122-0041). 
and the Ministry of Education, Singapore, under its Academic Research Fund Tier 2 (Proposal ID:
T2EP20122-0041). 
Any opinions, findings and conclusions or recommendations expressed in this material are those of the author(s) and do not reflect the views of A*STAR or the Ministry of Education, Singapore.


% \clearpage

\bibliographystyle{IEEEtran}
\bibliography{references.bib}

% \newpage

 

% \begin{IEEEbiography}
% [{\includegraphics[width=1in,height=1.25in,clip,keepaspectratio]{fig1}}]{Michael Shell}
% \end{IEEEbiography}


% \vspace{-10mm}

\begin{IEEEbiographynophoto}{Zhihao Wen}
received his Ph.D. degree in Computer Science from Singapore Management University in 2023. He is now a senior algorithm engineer in Ant Group. His research interests include graph-based machine learning and data mining, as well as their applications for the Web and social media.
\end{IEEEbiographynophoto}

\vspace{-10mm}
\begin{IEEEbiographynophoto}{Yuan Fang}
received his Ph.D. degree in Computer Science from University of Illinois at Urbana Champaign in 2014. He is currently an Associate Professor in the School of Computing and Information Systems, Singapore Management University. His research focuses on graph-based machine learning and data mining, as well as their applications for the Web and social media.
\end{IEEEbiographynophoto}

\vspace{-10mm}
\begin{IEEEbiographynophoto}{Pengcheng Wei}
received his Master degree in Computer Science from Hunan University in 2022. He is now a Ph.D. student in the Pillar of Information Systems Technology and Design, Singapore University of Technology and Design. His research interets include data mining, graph-based machine learning and system security.
\end{IEEEbiographynophoto}

\vspace{-10mm}
\begin{IEEEbiographynophoto}{Fayao Liu}
received the BEng and MEng degrees from the School of Computer
Science, National University of Defense Technology, Hunan, China, in 2008 and 2010, respectively, and the PhD degree from the University of Adelaide, Australia, in 2015. Her current research interests include machine learning and computer vision.
\end{IEEEbiographynophoto}

\vspace{-10mm}
\begin{IEEEbiographynophoto}{Zhenghua Chen}
received the B.Eng. degree in mechatronics engineering from University of Electronic Science and Technology of China (UESTC), Chengdu, China, in 2011, and Ph.D. degree in electrical and electronic engineering from Nanyang Technological University (NTU), Singapore, in 2017. He has been working at NTU as a research fellow. Currently, he is a scientist at Institute for Infocomm Research, Agency for Science, Technology and Research (A*STAR), Singapore. He has won several competitive awards, such as First Place Winner for CVPR 2021 UG2+ Challenge, A*STAR Career
Development Award, First Runner-Up Award for Grand Challenge at IEEE VCIP 2020, Finalist Academic Paper Award at IEEE ICPHM 2020, etc. He serves as Associate Editor for Elsevier Neurocomputing and Guest Editor for IEEE Transactions on Emerging Topics in Computational Intelligence. He is currently the Vice Chair of IEEE Sensors Council Singapore Chapter and IEEE Senior Member. His research interests include smart sensing, data analytics, machine learning, transfer learning and related applications.
\end{IEEEbiographynophoto}

\vspace{-10mm}
\begin{IEEEbiographynophoto}{Min Wu}
is a principal scientist in Machine Intellection Department, Institute for Infocomm Research. He received Ph.D. degree from Nanyang Technological University, Singapore, in 2011, and received B.S. degree in Computer Science from University of Science and Technology of China,
2006. His research interests include Graph Mining from Large-Scale Networks, Learning from Heterogeneous Data Sources, Ensemble Learning, and Bioinformatics.
\end{IEEEbiographynophoto}


\vfill

\end{document}